\documentclass[twoside]{article}
\usepackage[accepted]{aistats2022}

\usepackage[round]{natbib}

\usepackage{algorithm}
\usepackage{algorithmicx}
\usepackage[noend]{algpseudocode}
\usepackage{amsmath}
\usepackage{amssymb}
\usepackage{amsthm}
\usepackage{bbm}
\usepackage{bm}
\usepackage{caption}
\usepackage{color}
\usepackage{dirtytalk}
\usepackage{dsfont}
\usepackage{enumerate}
\usepackage{graphicx}
\usepackage{listings}
\usepackage{mathtools}
\usepackage{subfigure}
\usepackage{xspace}

\usepackage[usenames,dvipsnames]{xcolor}
\usepackage[bookmarks=false]{hyperref}
\hypersetup{
  pdffitwindow=true,
  pdfstartview={FitH},
  pdfnewwindow=true,
  colorlinks,
  linktocpage=true,
  linkcolor=Green,
  urlcolor=Green,
  citecolor=Green
}
\usepackage[capitalize,noabbrev]{cleveref}

\usepackage{tikz}
\usetikzlibrary{calc,trees,positioning,arrows}
\usetikzlibrary{automata}
\usetikzlibrary{trees}
\usetikzlibrary{decorations.pathmorphing} 
\usetikzlibrary{fit}         
\usetikzlibrary{backgrounds} 
\usetikzlibrary{shadows}
\usetikzlibrary{arrows.meta}
\tikzset{
  jump/.style={
     to path={
         let \p1=(\tikztostart),\p2=(\tikztotarget),\n1={atan2(\y2-\y1,\x2-\x1)} in
         (\tikztostart) -- ($($(\tikztostart)!#1!(\tikztotarget)$)!0.15cm!(\tikztostart)$)
         arc[start angle=\n1+180,end angle=\n1,radius=0.15cm] -- (\tikztotarget)}
  },
  jump/.default={0.5}
}

\tikzstyle{mynode}=[circle, thick, minimum size=0.7cm, draw=black!80]
\tikzstyle{observed}=[circle, thick, minimum size=0.7cm, draw=black!100, fill=black!20]
\tikzstyle{latent}=[circle, thick, minimum size=0.7cm, draw=black!80]
\tikzstyle{plate}=[rectangle, thick, inner sep=0.3cm, draw=black!100]
\tikzstyle{shadeplate}=[rectangle, thick, inner sep=0.4cm, draw=black!100]
\tikzstyle{table}=[circle,fill=blue!20,draw=black!100,inner sep=1pt, minimum size=30pt]
\tikzstyle{client}=[rectangle,fill=blue!20,draw=black!100,inner sep=1pt, minimum size=12pt]

\usepackage[backgroundcolor=white]{todonotes}

\def\Sset{\mathcal{S}}
\def\Aset{\mathcal{A}}
\def\Xset{\mathcal{X}}

\newcommand{\realset}{\mathbb{R}}

\newcommand{\indicator}[1]{\mathds{1}\!\left\{#1\right\}}
\newcommand{\E}[2]{\mathbb{E}_{#1} \left[#2\right]}
\newcommand{\prob}[2]{\mathbb{P}_{#1}\left(#2\right)}
\newcommand{\norm}[2]{\left\lVert#1\right\rVert_{#2}}
\newcommand{\abs}[1]{\left|#1\right|}

\newcommand{\determinant}[1]{\mathsf{det}\left(#1\right)}
\newcommand{\trace}[1]{\mathsf{trace}\left(#1\right)}

\newcommand{\probt}[1]{\mathbb{P}_t \left(#1\right)}
\newcommand*\dif{\mathop{}\!\mathrm{d}}

\newcommand{\mexp}{\ensuremath{\tt Exp4}\xspace}
\newcommand{\corralmexp}{\ensuremath{\tt CorralExp4}\xspace}
\newcommand{\mts}{\ensuremath{\tt MixTS}\xspace}
\newcommand{\ts}{\ensuremath{\tt TS}\xspace}
\newcommand{\lints}{\ensuremath{\tt LinTS}\xspace}
\newcommand{\units}{\ensuremath{\tt UniTS}\xspace}

\newcommand{\mmts}{\ensuremath{\tt mmTS}\xspace}

\def\Bregret{\mathcal{BR}}

\newtheorem{lemma}{Lemma}

\newtheorem{theorem}{Theorem}

\begin{document}

\twocolumn[

\aistatstitle{Thompson Sampling with a Mixture Prior}

\aistatsauthor{Joey Hong \And Branislav Kveton \And Manzil Zaheer}

\aistatsaddress{UC Berkeley$^*$ \And Amazon$^*$ \And Google DeepMind}

\aistatsauthor{Mohammad Ghavamzadeh \And Craig Boutilier}

\aistatsaddress{Google Research \And Google Research}]

\begin{abstract}
We study Thompson sampling (TS) in online decision making, where the uncertain environment is sampled from a \emph{mixture distribution}. This is relevant in multi-task learning, where a learning agent faces different classes of problems. We incorporate this structure in a natural way by initializing TS with a \emph{mixture prior}, and call the resulting algorithm \mts. To analyze \mts, we develop a novel and general proof technique for analyzing the concentration of mixture distributions. We use it to prove Bayes regret bounds for \mts in both linear bandits and finite-horizon reinforcement learning. Our bounds capture the structure of the prior, depend on the number of mixture components and their widths. We also demonstrate the empirical effectiveness of \mts in synthetic and real-world experiments.
\end{abstract}

\section{INTRODUCTION}
\label{sec:introduction}

\emph{Thompson sampling (TS)} \citep{ts} is arguably the most popular and practical class of exploration algorithms for \emph{stochastic bandits} \citep{bandit_book,ts} and \emph{reinforcement learning (RL)} \citep{rl_book,mdp_posterior_sampling}. However, in both settings, TS is almost exclusively applied with a \emph{unimodal prior} over model parameters \citep{ts,lints,mdp_posterior_sampling}. This is extremely limiting in a variety of settings, for instance, in a multi-task setting where a learning agent faces one of $L$ classes of bandit problems, each with a different distribution of model parameters. If this prior knowledge was expressed by a single unimodal distribution, it would generally be \say{wide} (hence uninformative), which can dramatically slow convergence of TS.

In this work, we incorporate \emph{mixture models} into TS for both stochastic bandits and RL. The idea behind mixture models is using latent variables to make a model more expressive \citep{bishop}.  In supervised learning, a more expressive model can better capture a complex population of sub-populations with similar features. For instance, \emph{Gaussian mixture models} (GMMs) \citep{clustering} are commonly used to cluster features in financial markets \citep{gmm_finance} and to identify classes of images \citep{bishop}. \emph{Topic models} \citep{lda}, which are mixtures of categorical distributions, are often used to analyze text data. Similarly, in online learning, algorithms can be more expressive by conditioning on a latent state \citep{hme}. In multi-task learning, where an agent faces a collection of tasks related through latent structure, we believe that TS can be substantially improved by using a more expressive prior.

We study TS with a \emph{mixture prior}, which is a joint probability distribution over an unobserved discrete latent state and model parameters. It is unclear \emph{a priori} if efficient algorithms exist for this problem class. From the computational perspective, the posterior distribution may not have a closed form; and thus may be hard to update efficiently. This is one reason why existing TS implementations use simple priors. Apart from computational issues, we might hope to exploit the problem structure to derive tighter regret bounds. The challenge is that the learning agent never observes the latent state. We address both challenges.

\renewcommand{\thefootnote}{\fnsymbol{footnote}}
\footnotetext[1]{The work was done while at Google Research.}
\renewcommand{\thefootnote}{\arabic{footnote}}

We make the following contributions. First, we propose a general algorithm, \emph{mixture Thompson sampling} (\mts), for a mixture prior with $L$ discrete latent states. \mts first samples a latent state from its posterior and then samples model parameters conditioned on that state. By explicitly modeling the latent state, the posterior can be efficiently maintained for common reward distributions, such as Bernoulli and Gaussian, using conjugacy. Second, we bound the $n$-round Bayes regret of \mts using a novel general analysis technique that accounts for jointly learning the model parameters and identifying the latent state; without ever observing it. We apply our technique in two settings: contextual linear bandit and finite-horizon RL. Finally, we evaluate \mts empirically in synthetic bandit and RL tasks, and in a task based on image classification using the CIFAR-100 dataset \citep{cifar}.

The main theoretical contribution of this work are the first Bayes regret bounds for TS with a mixture prior that are (i) sublinear in the number of rounds and (ii) depend on how informative the prior is. Specifically, the bounds depend on the structure of the prior, the number of mixture components and their width. When the prior is unimodal, $L = 1$, our bounds match the regret bounds of classical TS \citep{ts,lints}. On the other hand, when the mixture components have low width, the regret is determined by the cost of identifying the correct latent state \citep{latent_bandits_revisited}. \citet{latent_bandits_revisited} studied the same algorithm as \mts in bandits, but proved a linear regret bound for non-zero mixture component widths. In RL, we are the first to consider and analyze a mixture prior. 

\section{SETTING}
\label{sec:setting}

We consider an online decision-making problem where a learning agent interacts with an unknown environment sequentially over $n$ \emph{rounds}. We start with a multi-armed bandit setting and extend it to RL in \cref{sec:rl}. We adopt the following notation. Random variables are capitalized. The $i$-th entry of vector $v$ is $v_i$; if a vector $v_i$ is already indexed, then we denote its $j$-th entry by $(v_i)_j$. We use $\tilde{\mathcal{O}}$ for the big O notation up to logarithmic factors.

Our setting is defined as follows. In round $t \in [n]$, the agent takes an \emph{action} $A_t$ from an \emph{action set} $\Aset_t$ and observes \emph{reward} $Y_t \in \mathbb{R}$. The reward $Y_t$ is drawn i.i.d.\ from \emph{reward distribution} $P(\cdot \mid A_t; \theta)$. The distribution depends on the taken action $A_t$ and \emph{model parameters} $\theta \in \Theta$, where $\Theta$ is a set of feasible model parameters. We denote by $\mu_{\theta}(a) = \E{Y \sim P(\cdot \mid a; \theta)}{Y}$ the mean reward of action $a$ under model $\theta$, and assume that all rewards are $\sigma^2$-sub-Gaussian. We subscript the action set by $t$ as $\Aset_t$. This allows us to have changing action sets, which provides additional flexibility. Specifically, in contextual bandits, the context $X_t$ in round $t$ may influence which actions are possible, a dependence captured in $\Aset_t$. 

We denote by $\theta_*$ the true model parameters. In this work, we assume that $\theta_*$ is sampled from a \emph{mixture prior} $P_0$. The mixture prior is represented using a finite set of \emph{latent states} $\Sset$, where $|\Sset| = L$. Each latent state corresponds to a separate ``hypothesis'' for the parameter distribution. The model parameters $\theta_*$ are sampled as follows. First the true latent state is sampled as $S_* \sim P_0$ from the \emph{latent state prior}, then the model parameters are sampled as $\theta_* \sim P_0(\cdot \mid S_*)$ from the \emph{model parameter prior}. Formally, the distribution of $\theta_*$ is $\prob{}{\theta_* = \theta} = \sum_{s \in \Sset} P_0(\theta \mid s) P_0(s)$.

In multi-armed bandits, a typical goal is to maximize the expected $n$-round reward, or equivalently minimize the \emph{expected $n$-round regret}
$$
    \mathcal{R}(n; \theta_*) = \E{}{\sum_{t = 1}^n \mu_*(A_{t, *}) - \mu_*(A_t) \mid \theta_*}\,,
$$ 
where $\mu_*(a) = \mu_{\theta_*}(a)$ is the true mean reward of action $a$, $A_{t, *} = \max_{a \in \Aset_t} \mu_*(a)$ is the optimal action in round $t$, and the expectation is taken over both the randomness in the bandit algorithm and environment. Note that $\theta_*$ is fixed in $\mathcal{R}(n, \theta_*)$. In this work, we focus on an average performance over \emph{multiple problems}, each corresponding to different model parameters sampled from the prior. This is to capture the structure of the stochastic generative process in our analysis. By taking an expectation over $S_*$ and $\theta_*$, we obtain the \emph{$n$-round Bayes regret} \citep{russo_posterior_sampling}
$
    \Bregret(n) = \E{}{\mathcal{R}(n; \theta_*)}
$.

\section{ALGORITHM}
\label{sec:mixture_ts}

Thompson sampling \citep{ts,russo_posterior_sampling} is an algorithm that takes actions proportionally to being optimal under the posterior. This is usually implemented by first sampling model parameters $\theta_t$ from the posterior, then taking action $A_t = \arg\max_{a \in \Aset_t} \mu_{\theta_t}(a)$ that maximizes the mean reward under $\theta_t$. The posterior captures agent's uncertainty over the true model parameters $\theta_*$ conditioned on history. We denote the \emph{observation history} up to round $t$ by $H_t = (A_1, Y_1, \dots, A_{t-1}, Y_{t-1})$, and denote the respective conditional probability and expectation by $\prob{t}{\cdot} = \prob{}{\cdot \mid H_t}$ and $\E{t}{\cdot} = \E{}{\cdot \mid H_t}$.

Now we describe how Thompson sampling with a mixture prior works. We first note that the posterior over model parameters at round $t$ can be obtained by marginalizing over the latent state as
$
\probt{\theta_* = \theta} = 
\sum_{s \in \Sset} \probt{\theta_* = \theta \mid S_* = s}\probt{S_* = s} 
$.
Because of this structure, explicit modeling of the latent state allows for tractable sampling from and updates to the posterior. We denote the posterior by $P_t$, where $P_t(s) = \probt{S_* = s}$ and $P_t(\theta \mid s) = \probt{\theta_* = \theta \mid S_* = s}$. Sampling $\theta_*$ from the posterior is straightforward: first a latent state $S_t \sim P_t$ is sampled, then $\theta_t \sim P_t(\cdot \mid S_t)$ is sampled conditioned on $S_t$. We show that each component of the posterior can be computed tractably. 

The key insight is that each \emph{model parameter posterior} has form
$
P_t(\theta \mid s) \propto P_0(\theta \mid s) \prod_{\ell = 1}^{t-1} P(Y_t \mid A_t; \theta)
$,
and thus has a closed form when $P_0(\cdot \mid s)$ is conjugate to the reward distribution. This holds in many settings, such as Bernoulli rewards with a beta prior, and Gaussian rewards with a Gaussian prior. Moreover, we note that the latent state posterior can be written as
$
P_t(s) \propto P_0(s) \int_\theta P_t(\theta \mid s) \, d\theta
$.
The integral is the posterior predictive probability and can be computed efficiently when $P_0(\cdot \mid s)$ is conjugate to the reward distribution. The normalizing constant $\prob{}{H_t}$ is the same for all latent states. Since $\Sset$ is finite, we normalize $P_t(s)$ by dividing it with $\sum_{s = 1}^L P_t(s)$.

Based on the above, TS with a mixture prior can be implemented efficiently for many problems of interest. The resulting algorithm, \mts (\cref{alg:ts}), uses incremental posterior updates. In the bandit setting, our algorithm is an instance of \mmts \citep{latent_bandits_revisited} for latent bandits. Since \mts has a mixture prior, all model parameters live in the same parameter space, a key difference from \citet{latent_bandits_revisited} that allows us to analyze the concentration of the mixture posterior. In addition, we extend \mts to RL in \cref{sec:rl}.

\begin{algorithm}[t]
\caption{TS with a mixture prior (\mts)}\label{alg:ts}
\begin{algorithmic}[1]
  \State \textbf{Input:} Latent state prior $P_0$
  \State \phantom{\textbf{Input:}} model parameters priors $\{P_0(\cdot \mid s)\}_{s \in \Sset}$
  \State Initialize $P_1 \gets P_0$
  \For{$t \gets 1, \dots, n$}
    \State Sample $S_t \sim P_t$ and $\theta_t \sim P_t(\cdot \mid S_t)$
    \State Select $A_t \gets \arg\max_{a \in \Aset_t}\mu_{\theta_t}(a)$.
    \State Observe $Y_t$ and update 
    \State\quad 
    $
        P_{t+1}(\theta \mid s) \propto P_t(\theta \mid s)P(Y_t \mid A_t; \theta)\,, \ \forall s \in \Sset
    $
    \State\quad
    $
        P_{t+1}(s) \propto P_0(s) \int_{\theta} P_{t + 1}(\theta \mid s) \, d\theta
    $ 
  \EndFor
\end{algorithmic}
\end{algorithm}

\section{BAYES REGRET ANALYSIS}
\label{sec:analysis}

In this section, we prove a Bayes regret bound with a mixture prior. In \cref{sec:proof_sketch}, we provide a general analysis outline for \mts. We specialize it to contextual linear bandits in \cref{sec:lin_bandit} and extend it to RL in \cref{sec:rl}.

Bandit algorithms with latent variables are rare, and often lack a regret bound. The key step in our analysis is a novel construction of confidence intervals around latent variables. This is challenging because the latent variables are unobserved. Our analysis outline can be applied to any model, simply by specifying the confidence intervals. This shows the modularity and generality of our approach.

\subsection{General Analysis Outline}
\label{sec:proof_sketch}

Recall that $S_*$ and $\theta_*$ are the true latent state and model parameters, and let $\mu_*(a) = \mu_{\theta_*}(a)$. To simplify the sketch, we assume that $\mu_*(a) \in [0, 1]$; but \cref{thm:bayes_regret_linbandit_mixture} does not assume this.

Let $\bar{\mu}_t(a, s) = \E{\theta \sim P_t(\cdot \mid s)}{\mu_{\theta}(a)}$ be the posterior mean reward of action $a$ under latent state $s$, and $\sigma_t(a, s)$ be a high-probability confidence width for the model parameter posteriors $P_t(\cdot \mid s)$, that is $\prob{t}{\abs{\mu_*(a) - \bar{\mu}_t(a, s)} \geq \sigma_t(a, s)} \leq 1/n$. At a high level, our Bayes regret bounds include two terms. The first is due to concentration of the model parameter posteriors, and is bounded by the sum of confidence widths $\sum_{t = 1}^n \sigma_t(A_t, S_t)$. The second captures the identification of the latent state, and scales with $\sqrt{L n}$.

Let $A_{t, *} = \max_{a \in \Aset_t} \mu_*(a)$ be the optimal action in round $t$. From \citet{russo_posterior_sampling}, we can write the Bayes regret as
\begin{align}
\label{eq:bayes_regret_decomposition}
    \Bregret(n) 
    &= \E{}{\sum_{t = 1}^n\E{t}{\mu_*(A_{t,*}) - \bar{\mu}_t(A_{t, *}, S_*)}} + {} \\
    &\qquad \E{}{\sum_{t = 1}^n\E{t}{\bar{\mu}_t(A_t, S_t) - \mu_*(A_t)}}\,, \nonumber
\end{align}
where we use that $\bar{\mu}_t$ is a deterministic function of history $H_t$, and that $A_t, S_t$ and $A_{t,*}, S_*$ are i.i.d.\ given $H_t$. To bound the Bayes regret, we can bound each term individually as follows.

\textbf{Step 1.}
Bound the first term of \eqref{eq:bayes_regret_decomposition}. For round $t$, let event
\begin{align*}
    E_t = \left\{\forall a \in \Aset: \, \abs{\mu_*(a) - \bar{\mu}_t(a, S_*)} \leq \sigma_t(a, S_*)\right\}
\end{align*} 
 denote that the true mean is close to the posterior mean. Then
\begin{align}
\label{eq:bayes_regret_first_term}
    &\E{t}{\mu_*(A_{t,*}) - \bar{\mu}_t(A_{t, *}, S_*,)} \\
    &\leq
    \E{t}{(\mu_*(A_{t,*}) \!-\! \bar{\mu}_t(A_{t, *}, S_*))\indicator{\bar{E}_t}} \!+\!
    \E{t}{\sigma_t(A_{t, *}, S_*)} \nonumber
\end{align} 
where
$
    (\mu_*(A_{t,*}) - \bar{\mu}_t(A_{t, *}, S_*))\indicator{E_t}
    \leq \sigma_t(A_{t, *}, S_*)
$
is by definition of $E_t$. The first term of \eqref{eq:bayes_regret_first_term} can be bounded using the fact that event $\bar{E}_t$ is unlikely conditioned on $H_t$. The second term can be rewritten as $\E{t}{\sigma_t(A_{t, *}, S_*)} = \E{t}{\sigma_t(A_t, S_t)}$, using that $A_t, S_t$ and $A_{t,*}, S_*$ are i.i.d.\ conditioned on $H_t$. Finally, we sum over all rounds $t \in [n]$.

\textbf{Step 2.}
We want to bound the second term of \eqref{eq:bayes_regret_decomposition}. To do so, we first need to define \emph{confidence sets} over latent states. Formally, for each round $t$, we construct $C_t$ such that $S_* \in C_t$ holds with a high probability. Since the latent state is unobserved, we use a frequentist construction with a proxy statistic for how well the model parameter posterior of each latent state predicts the rewards. Let $N_t(s) = \sum_{\ell = 1}^{t-1} \indicator{S_\ell = s}$ be the number of times $s$ was sampled from posterior up to round $t$, and
\begin{align*}
    G_t(s) = \sum_{\ell = 1}^{t-1} \indicator{S_\ell = s} (\bar{\mu}_{\ell}(A_\ell, s) - \eta \sigma_{\ell}(A_\ell, s) - Y_\ell)
\end{align*}
be the total reward \say{excess} with respect to the posterior mean, where $\eta \in \mathbb{R}, \eta > 0$ is a scaling factor. Let $C_t = \{s \in \Sset: G_t(s) \leq \varepsilon\}$ be the set of latent states with at most $\varepsilon$ excess. 
We want to prove that $S_*$ lies in $C_t$ in round $t$ with a high probability,
\begin{align*}
\prob{}{\bigcup_{t=1}^n \{S_* \not\in C_t\}}
\leq 
\sum_{t = 1}^n \prob{}{S_* \not\in C_t}
=
\mathcal{O}(1) \,.    
\end{align*}
The key idea in the proof is that each $\bar{\mu}_t(A_\ell, S_*) - \eta \sigma_t(A_\ell, S_*) - Y_\ell < 0$ holds with a high probability conditioned on any history $H_t$, since we subtract the reward from its lower confidence bound. Since $\mu_*(A_\ell)$ is unknown, we substitute it with reward $Y_\ell$. We set $\varepsilon = \mathcal{O}(\sqrt{N_t(s) \log{n}})$ in $C_t$ to correct for reward noise. \citet{latent_bandits_revisited} consider a similar construction, but used prior means and widths. We achieve better regret bounds by using the posterior.

\textbf{Step 3.}
Now, we are ready to bound the second term of \eqref{eq:bayes_regret_decomposition}. Since regret at any round is trivially bounded by $1$, we have
\begin{align*}
    &\E{}{\sum_{t = 1}^n\bar{\mu}_t(A_t, S_t) - \mu_*(A_t)} \leq \E{}{\sum_{t=1}^n \indicator{S_t \not\in C_t}} + {} \\
    &\qquad 
    \E{}{\sum_{t = 1}^n\left(\bar{\mu}_t(A_t, S_t) - \mu_*(A_t) \right)\indicator{S_t \in C_t}}\,.
\end{align*}
Note that the first term can be bounded as
\begin{align*}
    \E{}{\sum_{t=1}^n \indicator{S_t \not\in C_t}} 
    &=
    \sum_{t = 1}^n \E{}{\prob{t}{S_t \not\in C_t}} \\ 
    &= 
    \sum_{t = 1}^n \E{}{\prob{t}{S_* \not\in C_t}} \\
    &= 
    \sum_{t = 1}^n \prob{}{S_* \not\in C_t} 
    = \mathcal{O}(1)\,,
\end{align*}
where we use that $S_t$ and $S_*$ are i.i.d.\ conditioned on $H_t$ for the first equality, and the bound derived in Step 2 for the second. Finally, we have
\begin{align}
\label{eq:bayes_regret_second_term}
    &\E{}{\sum_{t = 1}^n\left(\bar{\mu}_t(A_t, S_t) - \mu_*(A_t)\right)\indicator{S_t \in C_t}} \\
    &\,\leq \eta \E{}{\sum_{t = 1}^n\sigma_t(A_t, S_t)} 
    + {} \nonumber \\
    &\quad \E{}{\sum_{t = 1}^n\left(\bar{\mu}_t(A_t, S_t) - \eta \sigma_t(A_t, S_t) - Y_t\right) \indicator{S_t \in C_t}}
    \nonumber \,,
\end{align}
where we use that $\E{t}{Y_t \mid A_t, \theta_*} = \E{t}{\mu_*(A_t)}$. The first term of \eqref{eq:bayes_regret_second_term} is a sum of confidence widths, which decrease over time as the posterior concentrates. The second term of \eqref{eq:bayes_regret_second_term} can be bounded by the sum of the excesss $\sum_{s \in \Sset} G_{n + 1}(s)$, which is bounded by $\mathcal{O}(\sqrt{Ln \log{n}} + L)$ after we trivially bound the regret in the last round where each latent state is sampled. This is because in the last round $t$ where $S_t = s$, it must be true that $s \in C_t$, and thus $G_t(s)$ is bounded.




\subsection{Linear Bandits}
\label{sec:lin_bandit}

The above general analysis technique can be applied in various settings. Here we specialize it to a linear bandit with $d$ dimensions. In each round $t \in [n]$, a learning agent has a potentially changing action set $\Aset_t \subseteq \realset^d$ and takes action $A_t \in \Aset_t$. The agent observes reward $Y_t = A_t^\top\theta_* + \eta_t$, where $\theta_* \in \mathbb{R}^d$ is the unknown model parameter vector and $\eta_t \sim \mathcal{N}(0, \sigma^2)$ is a Gaussian noise. We assume that $\norm{a}{2} \leq \kappa$ for all rounds $t$ and $a \in \Aset_t$.

The prior is a mixture with $L$ components, indexed by latent states $s \in \Sset$. For each $s$, the model parameter prior is a Gaussian $P_0(\cdot \mid s) = \mathcal{N}(\cdot; \theta_{0, s}, \Sigma_{0, s})$, and we assume that $\theta_{0, s}$ is bounded as $\norm{\theta_{0, s}}{2} \leq 1$. This is a weaker assumption than in prior works, which typically assume that $\norm{\theta_*}{2}$ is bounded \citep{linucb,russo_posterior_sampling}. In round $t$, \mts samples $S_t \sim P_t$ and then $\theta_t \sim \mathcal{N}(\bar{\theta}_{t, S_t}, \Sigma_{t, S_t})$. Here $\bar{\theta}_{t, s}$ and $\Sigma_{t, s}$ are the posterior mean model parameter and its covariance, respectively, under latent state $s$ and are defined as
\begin{align}
  \Sigma_{t, s} &= (\Sigma_{0, s}^{-1} + \sigma^{-2} V_t)^{-1}\,, \nonumber \\
  \bar{\theta}_{t, s} &= \Sigma_{t,s}(\Sigma_{0, s}^{-1} \theta_{0, s} + \sigma^{-2} B_t)\,,
  \label{eq:mixture_posterior_linbandit}
\end{align}
where $V_t = \sum_{\ell=1}^{t-1} A_\ell A_\ell^\top$ and $B_t = \sum_{\ell=1}^{t-1} A_\ell Y_\ell$. The posterior mean reward of action $a$ and its confidence width are given by 
\begin{align*}
    \bar{\mu}_t(a, s) = a^\top\bar{\theta}_{s, t}, \
    \sigma_t(a, s) = \sqrt{2d\log(dn)} \norm{a}{\Sigma_{t, s}}\,.
\end{align*} 
We can bound the Bayes regret of \mts in this setting using the technique in \cref{sec:proof_sketch}. The proof is in \cref{sec:lin_bandit_proofs} and we state the bound below.

\begin{theorem}
\label{thm:bayes_regret_linbandit_mixture} Let $\lambda_{0, \max} = \max_{s \in \Sset}\lambda_{\max}(\Sigma_{0, s})$, where $\lambda_{\max}(\Sigma_{0, s})$ is the maximum eigenvalue of $\Sigma_{0, s}$ for latent state $s$. Let $\max_{a \in \Aset_t}\norm{a}{2} \leq \kappa$ hold in all rounds $t \in [n]$. Then the $n$-round Bayes regret of \mts is bounded as 
\begin{align}
  \Bregret(n) &\!\leq\!
  6\sigma d\sqrt{c_1 n \log(dn)} + 2\sigma\sqrt{Ln \log n} + c_2 \,,
\label{eq:bayes_regret_linbandit_mixture}
\end{align}
where
\begin{align*}
    c_1 = \left(1 + \frac{\kappa^2\lambda_{0, \max}}{\sigma^2} \right) \log\left(1 \!+\! \frac{\kappa^2\lambda_{0, \max} \ n}{\sigma^2 d}\right)\,,
\end{align*}
and $c_2$ is poly-logarithmic in $n$. 
\end{theorem}

\subsection{Discussion}

The bound has two main components: the regret for learning model parameters (Term $1$) under the assumption that the latent state is known, and the regret for identifying the latent state (Term $2$). Term $1$ is $\tilde{\mathcal{O}}(d \sqrt{c_1 n})$ and is of the same order as in linear TS ~\citep{russo_posterior_sampling}. The key difference is a prior-dependent constant $c_1$. Through $c_1$, Term 1 is linear in the maximum component width of the mixture prior $\sqrt{\lambda_{0, \max}}$. Term $2$ is $\tilde{\mathcal{O}}(\sqrt{L n})$ and is of the same order as identifying the true latent state among $L$ known models \citep{latent_bandits_revisited}.

Our bound does not depend on the latent state prior $P_0$. This is a shortcoming of our analysis, which constructs worst-case confidence sets for latent states, and is frequentist in this respect. We defer refinements of the analysis to future work. Another shortcoming is that we do not provide a matching lower bound. Although a Bayes regret lower bound exists for $K$-armed bandits \citep{lai87adaptive}, it is unclear how to apply it to structured problems. Seminal works on Bayes regret minimization \citep{russo_posterior_sampling,information_theory_posterior_sampling} also only derive upper bounds. We view deriving lower bounds as another avenue for future work.

Our analysis improves upon that of \citet{latent_bandits_revisited} by analyzing concentration of the model parameter posteriors. We attain $\tilde{\mathcal{O}}(d \sqrt{c_1 n} + \sqrt{L n})$ regret that is fully sublinear in $n$. In contrast, \citet{latent_bandits_revisited} have a regret bound $\tilde{\mathcal{O}}(c' n + \sqrt{L n})$, where $c'$ is a constant proportional to the maximum component width $\sqrt{\lambda_{0, \max}}$. This is because their analysis is agnostic to posterior improvements and treats prior uncertainty as a penalty, resulting in a linear regret bound.

Another natural comparison is to TS without the mixture prior. Since Bayes regret bounds are proved under the assumption of a correct prior, there are no other comparable Bayes regret bounds. However, we can compare to frequentist worst-case regret bounds, which hold even when the prior is misspecified. A state-of-the-art regret bound for \lints is $\tilde{\mathcal{O}}(d^{3 / 2} \sqrt{n})$ \citep{abeille17lints}. In contrast, our bound is $\tilde{\mathcal{O}}(d \sqrt{c_1 n} + \sqrt{L n})$, where $c_1$ scales with the maximum component width of the mixture prior $\sqrt{\lambda_{0, \max}}$ and $L$ denotes the number of latent states. With a sufficiently informative prior, $c_1 < d$; and with a small number of mixture components, $\sqrt{L} < d^{3 / 2}$; our bound improves over frequentist regret bounds for \lints.

\section{FINITE-HORIZON RL}
\label{sec:rl}

Next we extend our results to \emph{reinforcement learning (RL)} \citep{rl_book} in \emph{finite-horizon Markov decision processes (MDPs)} \citep{mdps}. First, we formalize RL with a mixture prior. Then, in \cref{sec:proof_sketch_rl}, we extend the general analysis outline from \cref{sec:proof_sketch}. Finally, in \cref{sec:tabular_mdp}, we apply the outline to derive a Bayes regret bound for \mts in a finite-horizon tabular MDP.

We have $n$ \emph{episodes} indexed by $t \in [n]$. In each episode, a learning agent interacts with an MDP for $h$ \emph{steps}. We refer to $h$ as the \emph{horizon}. We denote a finite-horizon MDP by $M = (\Xset, \Aset, R, T, h, \rho)$, where $\Xset$ is the state space, $\Aset$ is the action space, $R_M(x, a) \in [0, 1]$ is the mean reward when selecting action $a$ in state $x$, $T_M(x, a, x') = \prob{}{X_{i+1} = x' \mid X_i = x, A_i = a; M}$ is the probability of transitioning to state $x'$ if action $a$ is taken at state $x$, $h$ is the horizon, and $\rho$ the initial state distribution. We consider the special case of \emph{tabular MDPs}, where both $\Xset$ and $\Aset$ are finite sets. As a shorthand, let $T_M(x, a) = (T_M(x, a, x'))_{x' \in \Xset}$ be a vector for all transitions.

A policy $\pi = (\pi^i)_{i = 1}^h$ is a vector, one per step, where each $\pi^i: \Xset \to \Aset$ maps states to actions. We define the value of policy $\pi$ in MDP $M$ as
$
    V_M(\pi) = \E{}{\sum_{i = 1}^{h} R_M(X_i, A_i) \mid M, \pi}
$,
where $X_1 \sim \rho$, $A_i = \pi^i(X_i)$, and $X_{i+1} \sim \mathrm{Cat}(\cdot \mid T_M(X_i, A_i))$. The value is the expected total reward of acting under $\pi$ in $M$.

Let $M_*$ be the true MDP and $\pi_*$ be the optimal policy $\pi_* = \arg\max_{\pi} V_{M_*}(\pi)$ \citep{mdp_optimal_policy}.
We assume that $M_*$ is sampled hierarchically from a mixture prior $P_0$: first a latent state $S_* \sim P_0$ is sampled, then the MDP $M_* \sim P_0(\cdot \mid S_*)$. 
This generalizes prior work on TS in RL \citep{mdp_posterior_sampling,mdp_posterior_sampling_worst_case}, where the mixture prior is not considered. Recently, \citet{mixture_mdp} studied MDPs whose mean rewards and transition probabilities are linear mixtures, but assume the mean rewards and probabilities per component are known.
As a shorthand, we subscript by $*$ to denote statistics related to the true MDP $M_*$, such as $V_* = V_{M_*}$, and equivalently for $R_*$ and $T_*$. 
The Bayes regret of an algorithm over $n$ episodes is given by
$
    \Bregret(n) = \E{}{\sum_{t = 1}^n V_{*}(\pi_*) - V_{*}(\pi_t)}
$,
where $\pi_t$ is the policy chosen by the algorithm in episode $t$, and the randomness is over MDP $M_*$, policies selected by the learning agent, and observations. The history is given by $H_t = ((X_{\ell, i}, A_{\ell, i}, R_{\ell, i}))_{i \in [h], \, \ell \in [t - 1]}$, where $X_{\ell, i}, A_{\ell, i}, R_{\ell, i}$ are the state, action and reward for step $i$ of episode $\ell$. The reward of an episode is $Y_t = \sum_{i = 1}^h R_{t, i}$.

\subsection{General Analysis Outline}
\label{sec:proof_sketch_rl}

In finite-horizon RL, \mts operates as \cref{alg:ts}, but with MDP $M_t$ instead of parameters $\theta_t$ and policy $\pi_t$ instead of action $A_t$. That is, \mts in episode $t$ first samples latent state $S_t \sim P_t$, then the MDP conditioned on the sampled latent state $M_t \sim P_t(\cdot \mid S_t)$. Finally, the chosen policy in episode $t$ maximizes the value $\pi_t = \arg\max_{\pi}V_{M_t}(\pi)$. This algorithm is a generalization of PSRL \citep{mdp_posterior_sampling}, where a mixture prior is used. While bandit analyses can be often adapted to RL, we make a notable deviation. Prior works construct confidence intervals for each state of an MDP \citep{mdp_posterior_sampling,mdp_information_theory}. This cannot be done with latent variables, which are shared by all states. Therefore, we construct the intervals over entire MDP trajectories.

For episode $t$, let $\overline{V}_t(\pi, s) = \E{M \sim P_t(\cdot \mid s)}{V_M(\pi)}$ be the expected value of policy $\pi$ conditioned on $s$ and $H_t$. We have the following Bayes regret decomposition,
\begin{align}
\label{eq:bayes_regret_decomposition_mdp}    
    \Bregret(n)
    &=
    \E{}{\sum_{t = 1}^n \E{t}{V_*(\pi_*) - \overline{V}_t(\pi_*, S_*)}} + {}  \\
    &\quad \E{}{\sum_{t = 1}^n \E{t}{\overline{V}_t(\pi_t, S_t) -  V_*(\pi_t)}} \,, \nonumber
\end{align}
where we use that $S_t, \pi_t$ are distributed identically to $S_*, \pi_*$ conditioned on $H_t$.

The proof sketch is similar to the one in \cref{sec:proof_sketch}, but differs in two notable aspects. We list the main differences and defer the full sketch to \cref{sec:mdp_proofs}. First, the expected value of a policy under a latent state $\overline{V}_t(\pi, s)$ is used in place of the mean reward $\bar{\mu}_t(a, s)$. Second, in order to construct a confidence interval around $\overline{V}_t(\pi, s)$, we use the sum of confidence widths over steps of a trajectory.
Specifically, for any policy $\pi$, we have with high probability,
\begin{align*}
    &V_{M_t}(\pi) - \overline{V}_t(\pi, s)
    = \E{M \sim P_t(\cdot \mid s)}{V_{M_t}(\pi) - V_M(\pi)} \\
    &\,\leq \E{t}{h \sum_{i = 1}^h c_t(X_{t, i}, A_{t, i}, s) + \phi_t(X_{t, i}, A_{t, i}, s)}\,,
\end{align*} 
where we use the value difference lemma \citep{mdp_posterior_sampling}.
Here, we define a high-probability confidence intervals around the mean reward and transition probabilities, $c_t(x, a, s)$ and $\phi_t(x, a, s)$, respectively, for all state-action pairs $x, a$.  For $\bar{r}_t(x, a, s) = \E{M \sim P_t(\cdot \mid s)}{R_M(x, a)}$ as the posterior mean reward, we have $\prob{t}{\abs{R_M(x, a) - \bar{r}_t(x, a, s)} \geq c_t(x, a, s)} \leq 1/n$. Similarly, for $\bar{p}_t(x, a, x', s) = \E{M \sim P_t(\cdot \mid s)}{T_M( x, a, x')}$ as the posterior mean transition probability to state $x'$, and $\bar{p}_t(x, a, s)$ as a vector of such probabilities over all states $x' \in \Xset$, we have $\prob{t}{\norm{T_M(x, a) - \bar{p}_t(x, a, s)}{1} \geq \phi_t(x, a, s)} \leq 1/n$. The sum over $c_t(X_{t, i}, A_{t, i}, s)$ and $\phi_t(X_{t, i}, A_{t, i}, s)$ is used in place of $\sigma_t(A_t, S_t)$.

\subsection{Finite-Horizon Tabular MDPs}
\label{sec:tabular_mdp}

We consider finite-horizon tabular MDPs $M$ with Bernoulli rewards. In particular, for step $i$ of episode $t$, reward $R_{t, i}$ is sampled from a Bernoulli with mean $R_M(X_{t, i}, A_{t, i})$.

Recall that MDP $M = (\Xset, \Aset, R, T, h, \rho)$ has both mean rewards and transition probabilities. Let $R_M = (R_M(x, a))_{x, a}$ and $T_M = (T_M(x, a))_{x, a}$ be their respective concatenations across all state-action pairs. For true MDP $M_*$, which is unknown to the learning agent, let $R_*, T_*$ be these quantities. We consider the following generative process in sampling $M_*$. First a latent state $S_* \sim P_0$ is sampled. Then, the mean reward for state-action $x, a$ follows a beta prior $R_*(x, a) \sim \mathrm{Beta}(\alpha_{0, S_*}^R(x, a))$ with $\alpha_{0, s}^R(x, a) \in \mathbb{R}_{+}^2$ for any latent state $s$, and the transition probabilities follow a Dirichlet prior $T_*(x, a) \sim \mathrm{Dir}(\alpha_{0, S_*}^T(x, a))$ with $\alpha_{0, s}^T(x, a) \in \mathbb{R}_{+}^{|\Xset|}$. Here $\mathbb{R}_{+}$ denotes the space of positive reals. Finally, $M_* = (\Xset, \Aset, R_*, T_*, h, \rho)$ uses these sampled quantities.

Recall that in episode $t \in [n]$, \mts samples latent state $S_t \sim P_t$, then MDP $M_t \sim P_t(\cdot \mid S_t)$. Sampling $M_t$ consists of independently sampling, for each $x, a$, mean rewards $R_{M_t}(x, a) \sim \mathrm{Beta}(\alpha_{t, S_t}^R(x, a))$ and transition probabilities $T_{M_t}(x, a) \sim \mathrm{Dir}(\alpha_{t, S_t}^T(x, a))$. For latent state $s$, we denote by
$\alpha_{t, s}^R(x, a), \ \alpha_{t, s}^T(x, a)$
the parameters of the respective Dirichlet posteriors. Specifically,
\begin{align}
\label{eq:mixture_posterior_tabular_mdp_r}
    \bar{r}_t(x, a, s) &= \frac{(\alpha^R_{t, s}(x, a))_1}{\norm{\alpha^R_{t, s}(x, a)}{1}}\,,\\
    c_t(x, a, s) &= \sqrt{\frac{2\log(2|\Xset||\Aset|n)}{\norm{\alpha^R_{t, s}(x, a)}{1} + 1}}\,, \nonumber
\end{align}
are the posterior mean and confidence width for the mean reward under $x, a$. Similarly, we have
\begin{align}
\label{eq:mixture_posterior_tabular_mdp_p}
    \bar{p}_t(x, a, x', s) &= \frac{(\alpha^T_{t, s}(x, a))_{x'}}{\norm{\alpha^T_{t, s}(x, a)}{1}} \,,\\
    \phi_t(x, a, s) &= \sqrt{\frac{4|\Xset|\log(4|\Xset||\Aset|n)}{\norm{\alpha^T_{t, s}(x, a)}{1} + 1}}\,, \nonumber
\end{align}
for the transition probabilities. We simply state the Bayes regret bound and defer a full proof to \cref{sec:mdp_proofs}.

\begin{theorem}
\label{thm:bayes_regret_mdp_mixture} Let 
$$
\Lambda_{0, s} = \min\left\{\min_{x, a}\norm{\alpha_{0, s}^R(x, a)}{1}, \min_{x, a}\norm{\alpha_{0, s}^T(x, a)}{1}\right\}
$$
represent how concentrated the reward and transition priors are for latent state $s$, where higher values correspond to lower prior widths. Let $\Lambda_{0, \min} = \min_{s \in \Sset} \Lambda_{0, s}$. Then the $n$-episode Bayes regret of \mts is bounded
\begin{align*}
    &\Bregret(n) \leq \\
    &\ 6|\Xset|h^{3/2}\sqrt{c_1|\Aset|n \log(4|\Xset||\Aset|n)}
    \!+\! \sqrt{Lhn\log{n}} \!+\! c_2 \,.
\end{align*}
where
\begin{align*}
    c_1 = \log\left(1 + \frac{hn}{2 |\Xset||\Aset|\Lambda_{0, \min}}\right)\,,
\end{align*}
and $c_2$ is poly-logarithmic in $n$. 
\end{theorem}

Similarly to \cref{thm:bayes_regret_linbandit_mixture}, the above regret bound decomposes into the regret due to learning the MDP under the assumption that the latent state is known (Term $1$), and the regret due to identifying the correct latent state (Term $2$). Term $1$ is $\tilde{\mathcal{O}}(|\Xset|h^{3/2}\sqrt{c_1|\Aset|n})$ and matches classical TS bounds \citep{mdp_posterior_sampling}. The prior width is captured by $\Lambda_{0, \min}$ in $c_1$, which represents the minimum pseudo-counts in our beta and Dirichlet priors. Roughly speaking, the variance of beta and Dirichlet distributions is bounded by the reciprocal of these counts \citep{dirichlet_concentration}. So, when $\Lambda_{0, \min}$ is large, the beta and Dirichlet priors over mean rewards and transitions have low widths. Through $c_1$, Term $1$ goes to zero in this regime. Then the regret is dominated by Term $2$, which is $\tilde{\mathcal{O}}(\sqrt{Lhn})$ for identifying the correct latent state.

\section{EXPERIMENTS}
\label{sec:experiments}

We evaluate \mts in a synthetic and real-world problems. The goals of our experiments are the following: (1) assess the degree to which the Bayes regret bounds in \cref{thm:bayes_regret_linbandit_mixture,thm:bayes_regret_mdp_mixture} match the actual regret, (2) show that \mts outperforms TS with a less-informative unimodal prior and other online model selection algorithms in a challenging real-world problem, and (3) show that \mts still performs well when extended to RL settings.

\subsection{Synthetic Linear Bandit}

We begin with a synthetic $d$-dimensional Gaussian linear bandit where $d = 30$. We consider up to $L = 30$ latent states. The latent state prior is uniform, $P_0(s) = 1 / L$ for each $s$. The model parameter prior is an isotropic Gaussian $P_0(\cdot \mid s) = \mathcal{N}(\cdot; \theta_{0, s}, \sigma_0^2 I_d)$. The $i$-th entry of $\theta_{0, s}$ is $0.9$ when $i = s$, and $0.1$ otherwise. The action set is constant over all rounds $\Aset_t = \Aset \subseteq \mathbb{R}^d$ and consists of all $d$-dimensional indicator vectors. The reward for action $A_t$ is sampled from a Gaussian $Y_t \sim \mathcal{N}(A_t^\top \theta_*, \sigma^2)$ with $\sigma = 0.1$. The horizon is $n = 1,000$ rounds. We run $\mts$ $200$ times, with $S_*$ and $\theta_*$ sampled from the prior at the beginning of each run. We vary two quantities in \cref{thm:bayes_regret_linbandit_mixture}, the prior width $\sigma_0 = \lambda_{0, \max}$ and number of latent states $L$, and assess their effect on regret.

For each $\sigma_0$ and $L$, we use the mean regret over multiple runs, where in each run, model parameters are drawn as $\theta_* \sim P_0$, to approximate the Bayes regret. The regret is reported in \cref{fig:experiments}, together with the upper bound in \cref{thm:bayes_regret_linbandit_mixture}. The upper bound is multiplied by $1 / 30$, which changes the scale but preserves the shape. We observe that our bound correctly estimates the shape of the empirical regret as a function of $\sigma_0$. In a similar experiment, where $\sigma_0 = 0.05$ is fixed and we vary the number of latent states $L$, we again observe that our bound correctly estimates the shape of the empirical regret as a function of $L$. We conclude that \cref{thm:bayes_regret_linbandit_mixture} scales correctly with the parameters of our problem class.

\subsection{Image Classification}

In our second experiment, we consider an image classification problem with a mixture of high-level tasks. We use the CIFAR-100 dataset \citep{cifar}, which consists of $60,000$ images of size $32 \times 32$. There are $50,000$ training and $10,000$ test images. Each image belongs to one of $L = 100$ classes (image labels). 

We treat each class as a \emph{task}, so that images in class $s$ have high reward when the task is $s$. At the beginning of each run, a class is sampled as $S_* \sim P_0$, where $P_0(s) = 1 / L$ for all $s$. In round $t$, the action set $\Aset_t$ consists of $10$ randomly chosen images from the CIFAR-100 test set, where one image is guaranteed to be from class $S_*$. The reward of an image from class $S_*$ is $\mathrm{Ber}(0.9)$ and for all other classes is $\mathrm{Ber}(0.1)$. The horizon is $n = 500$ rounds. For such short horizons, the effect of the prior is more noticeable. We cast this problem as a linear bandit with features from a state-of-the-art EfficientNet-L2 network~\cite{Xie_2020_CVPR, efficientnet,foret2021sharpnessaware}. This is a convolutional neural network pretrained on both ImageNet~\citep{ILSVRC15} and unlabeled JFT-300M~\citep{jft} with input resolution $475$, and fine-tuned on the CIFAR-100 training set. Each action $a \in \Aset_t$ is a $100$-dimensional feature vector, the embedding after applying the network.

The mixture prior is obtained by clustering similar tasks from the CIFAR-100 training set. This is done as follows. First, we sample $1000$ random datasets of size $n = 500$ from the training set. For each dataset, we randomly choose the class $S_* \sim P_0$, and assign reward one to images from class $S_*$ and zero otherwise. Second, we fit a linear model to each dataset, where the image features are generated as above. Finally, we fit a GMM with $L$ components to the parameter vectors of the trained linear models, generating cluster means and covariances $(\theta_{0, s}, \Sigma_{0, s})_{s \in \Sset}$. The model parameter prior for $s$ is $P_0(\cdot \mid s) = \mathcal{N}(\cdot; \theta_{0, s}, \Sigma_{0, s})$.

We compare \mts to four baselines: \ts, \units, \mexp \citep{exp4}, and \corralmexp. 
\ts is Thompson sampling with an uninformative Gaussian prior $\mathcal{N}(\mathbf{0}, I_d)$ over model parameters. \units is TS with a unimodal Gaussian prior fit to the same data as the GMM. This baseline shows the importance of using mixtures, as opposing to just using past data. \mexp uses the prior means $(\theta_{0, s})_{s \in \Sset}$ as $L$ experts, where the action of expert $s$ is $\arg\max_{a \in \Aset_t} a^\top \theta_{0, s}$. The actions are a weighted vote of the experts, where better experts have higher weights. Finally, \corralmexp uses \mexp to track experts, but additionally adapts the parameters of each expert so that in round $t$, the action of expert $s$ is $\arg\max_{a \in \Aset_t} a^\top \bar{\theta}_{t, s}$, where $\bar{\theta}_{t, s}$ is defined as in \eqref{eq:mixture_posterior_linbandit}.  \corralmexp is an instance of a corralling bandit algorithm \citep{corralling_exp,corralling,corralling_2}, where a master (\mexp) switches between base algorithms (linear regressors). We measure the mean reward of each method, averaged over $100$ independent runs. Note that all TS algorithms are misspecified in this experiment, because the models are not linear and the reward noise is not Gaussian. We use $\sigma = 0.5$ since the rewards are in $[0, 1]$. As shown in \cref{fig:experiments}, \mts greatly outperforms \units and \ts, especially during the cold-start regime, due to using a strong mixture prior fitted to existing data. \mts also outperforms \mexp and \corralmexp by explicitly leveraging the latent state posterior to switch between models. Although \corralmexp uses the same model updates, it switches between the models using an adversarial algorithm.

\begin{figure*}[ht]
\centering
\begin{minipage}{0.32\textwidth}
    \includegraphics[width=\linewidth]{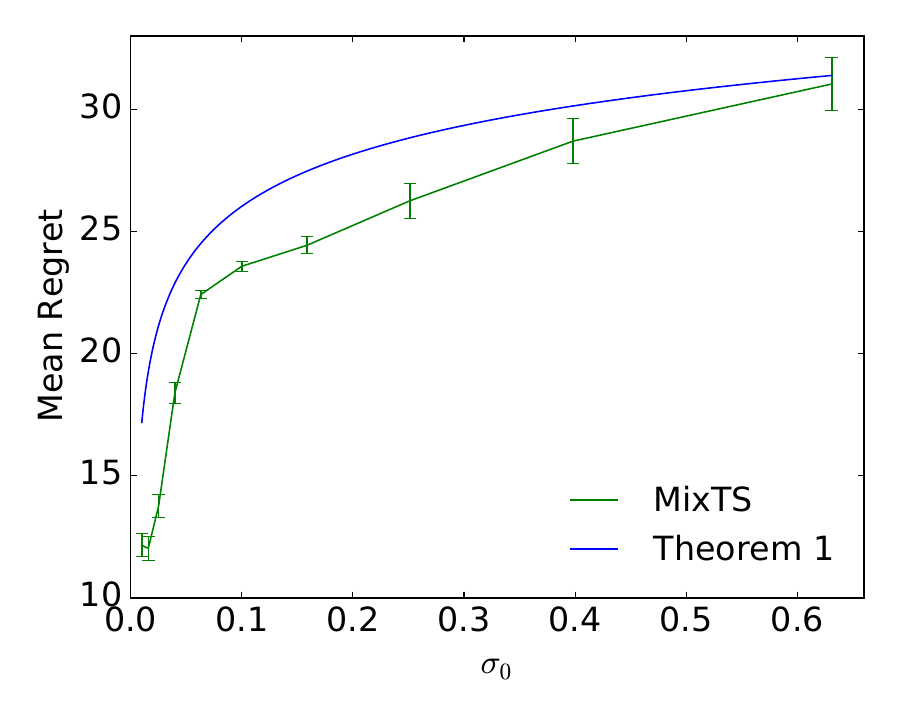}
\end{minipage}
\begin{minipage}{0.32\textwidth}
    \includegraphics[width=\linewidth]{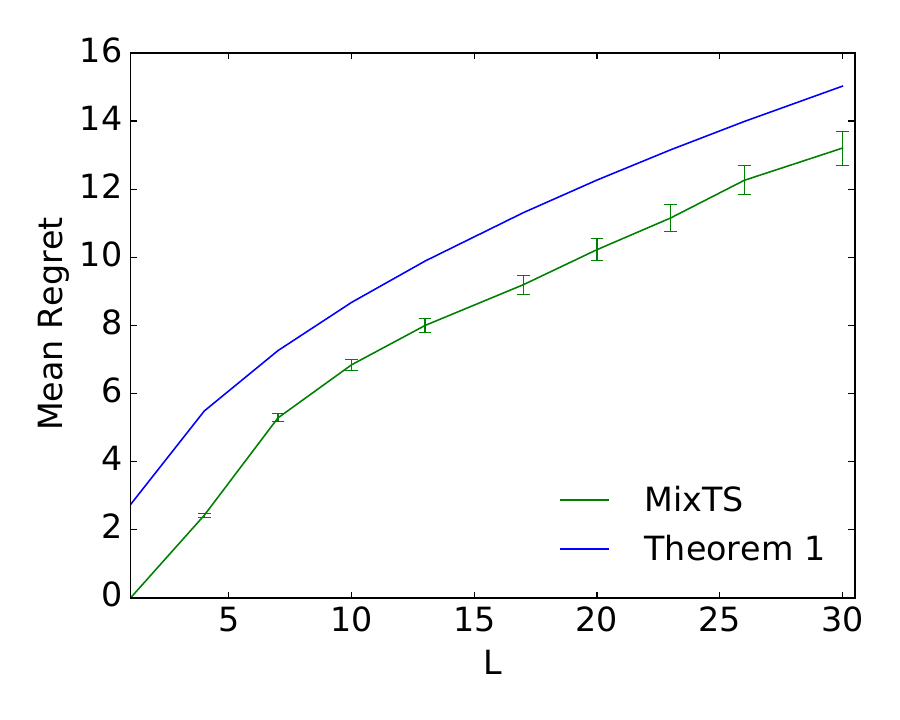}
\end{minipage}
\begin{minipage}{0.33\textwidth}
    \includegraphics[width=\linewidth]{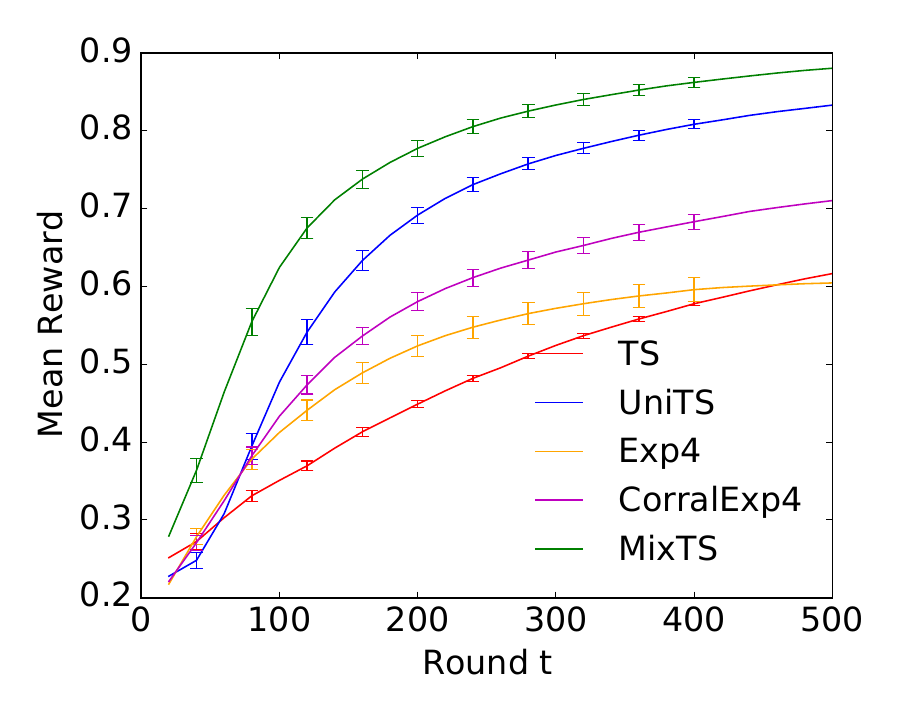}
\end{minipage}
\vspace{-3mm}
\caption{\emph{Left}: Bayes regret as function of prior width $\sigma_0$. \emph{Middle}: Bayes regret as function of the number of latent states $L$. \emph{Right}: Mean reward on a CIFAR-100 classification bandit.}
\label{fig:experiments}
\vspace{-0.1in}
\end{figure*}

\subsection{Synthetic MDP}
\label{sec:mdp_experiment}
In our final experiment, we consider a synthetic finite-horizon MDP based on the \emph{RiverSwim} environment \citep{mdp_posterior_sampling}. \emph{RiverSwim} consists of $|\Xset|$ states arranged in a chain. The agent starts at the state in the middle and at every time step, can choose to swim right or left, $|\Aset| = 2$. The environment is parameterized by a latent state that denotes the direction of the current, which can be right or left, $L = 2$. At a high level, swimming with the current is always successful, but swimming against the current likely fails. If the current is to the left, the agent receives a small reward for swimming left at the leftmost state, but receives a much larger reward for swimming right at the rightmost state; if the current is to the right, the opposite holds. The optimal policy involves swimming against the current to receive the large reward. The prior mean MDP when the current is to the left is shown in \cref{fig:riverswim}. The MDP when the current is to the right is symmetric.

In our experiments, we consider $|\Xset| = 10$ and horizon $h = 20$. The latent state prior is uniform, $P_0(s) = 1/2$ for $s$ as left or right. The MDP prior, conditioned on each latent state, consists of beta and Dirichlet priors for the mean reward and transition probabilities for each state-action pair $(x, a)$, such that the mean MDP under the prior matches the values in \cref{fig:riverswim}. The number of episodes is $n = 1,000$ episodes, and we run \mts $500$ times on independent samples of the MDP from the prior. In \cref{fig:riverswim}, we compare the mean regret over the $500$ runs of \mts against PSRL \citep{mdp_posterior_sampling}, which is a TS algorithm that uses a uniform prior over rewards and transitions. \mts greatly outperforms PSRL because it identifies the correct latent state, or direction of the current, much more quickly than PSRL learns the reward and transitions from scratch. 

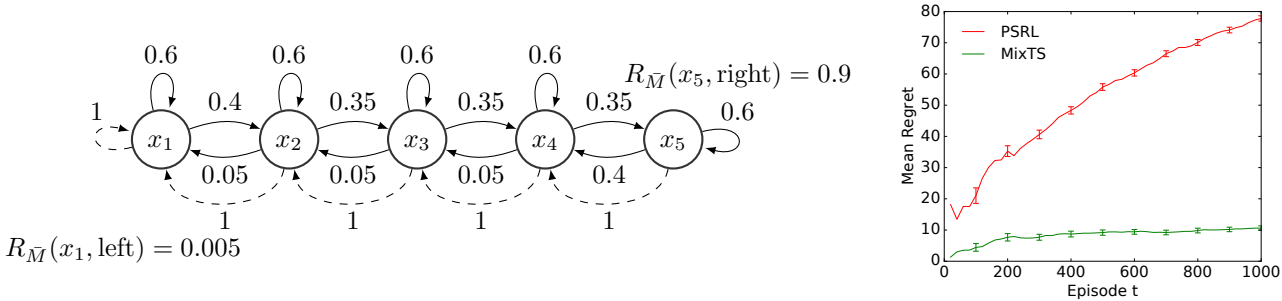
\begin{figure*}[ht]
    \centering
    \hfill
\begin{minipage}{0.64\textwidth}\hspace{-0.5in}
    \begin{tikzpicture}[>=latex,text height=1.5ex,text depth=0.25ex]
    \matrix[row sep=0.9cm,column sep=0.9cm] {
    \node (s1) [latent] {$x_1$}; &
    \node (s2) [latent] {$x_2$}; &
    \node (s3) [latent] {$x_3$}; &
    \node (s4) [latent] {$x_4$}; &
    \node (s5) [latent] {$x_5$}; &
    \\
    };
    \path[->]
     (s1) edge[loop above] node {0.6} (s1)
     (s1) edge[bend left=20] node[above] {0.4} (s2)
     (s1) edge[dashed, loop left, draw=black] node[above=1mm] {1} node[below=12mm]{$\qquad R_{\bar{M}}(x_1, \text{left}) = 0.005$} (s1)
     (s2) edge[bend left=20] node[below] {0.05} (s1)
     (s2) edge[dashed, bend left=80, draw=black] node[below] {1} (s1)
     (s2) edge[loop above] node {0.6} (s2)
     (s2) edge[bend left=20] node[above] {0.35} (s3)
     (s3) edge[bend left=20] node[below] {0.05} (s2)
     (s3) edge[dashed, bend left=80, draw=black] node[below] {1} (s2)
     (s3) edge[loop above] node {0.6} (s3)
     (s3) edge[bend left=20] node[above] {0.35} (s4)
     (s4) edge[bend left=20] node[above] {0.35} (s5)
     (s4) edge[dashed, bend left=80, draw=black] node[below] {1} (s3)
     (s4) edge[loop above] node {0.6} (s4)
     (s4) edge[bend left=20] node[below] {0.05} (s3)
     (s5) edge[loop right] node[above=1mm] {0.6} node[above=6mm] {$R_{\bar{M}}(x_5, \text{right}) = 0.9$} (s5)
     (s5) edge[bend left=20] node[below] {0.4} (s4)
     (s5) edge[dashed, bend left=80, draw=black] node[below] {1} (s4)
    ;
    \end{tikzpicture}
\end{minipage}
\hfill
\begin{minipage}{0.32\textwidth}
\includegraphics[width=\linewidth]{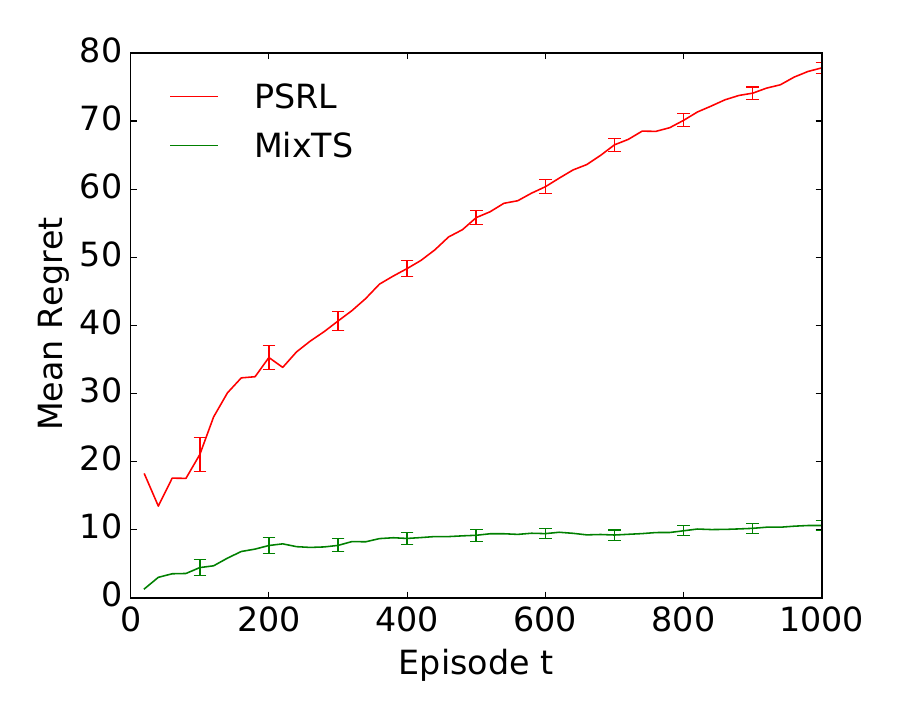}
\end{minipage}
\vspace{-3mm}
\caption{\emph{Left:} \emph{RiverSwim} with $|\Xset| = 5$ and current to the left. Solid and dashed arrows represent transitions under actions ``left" and ``right", respectively. Numbers denote the mean reward and transition probabilities of the mean MDP $\bar{M}$ under the prior. \emph{Right:} Mean regret on finite-horizon \emph{RiverSwim} environment.}
\label{fig:riverswim}
\end{figure*}

\section{RELATED WORK}
\label{sec:related_work}

\textbf{Thompson sampling.}
Thompson sampling is known for its computational efficiency and strong empirical performance \citep{ts,chapelle11empirical,lints}. \citet{russo_posterior_sampling} derived first Bayes regret bounds for TS in bandits and RL \citep{mdp_posterior_sampling}. We build on these works by considering a mixture prior. By explicitly modeling a latent state, we can implement TS efficiently, as well as derive improved prior-dependent Bayes regret bounds. Alternatively, information theory has been used to derive Bayes regret bounds \citep{information_theory_posterior_sampling,mdp_information_theory}. These proofs rely on the entropy of posterior distributions, which do not have closed forms for mixtures. Recent works applied approximate TS to complex structured problems \citep{ts_complex,ts_graphical_model}. Such algorithms are general, but can only be analyzed in limited settings with strong assumptions. We consider a special prior structure, and derive improved regret bounds for bandits and RL.

A related work on TS with mixture distributions is \citet{urteaga18mixture}. The setting of this work is completely different because they study a mixture reward distribution. In comparison, we study a mixture of model parameters. To make this distinction clear, consider a linear bandit. \citet{urteaga18mixture} would have non-Gaussian rewards sampled from a Gaussian mixture model (GMM). We would have Gaussian rewards with model parameters sampled from a GMM. More recently, \citet{bandit_gmm} proposed a non-parametric GMM over the rewards in the bandit setting. This is another instance of a mixture reward distribution.

\textbf{Online model selection.}
Our work is also related to online model selection, as each latent state corresponds to a different hypothesis for the distribution of the environment. Identifying the true latent state is analogous to selecting the best-performing base model. \mexp \citep{exp4} is one of the earliest algorithms for solving this problem in adversarial environments. Bayesian policy-reuse (BPR) \citep{bpr} could be used in stochastic environments but it does not have theoretical guarantees. More recently, in corralling bandits, a master algorithm learns the best-performing base bandit algorithm. \citet{corralling_exp} proposed a modified version of \mexp as the master. Corralling algorithms have also been extended to the stochastic setting \citep{corralling,corralling_2,model_selection_linbandit}. In all above works, the base algorithm is updated only when selected by the master. In our work, because of the full Bayesian treatment, all mixture components are always updated, which increases statistical efficiency. As shown in \cref{sec:experiments}, \mts outperforms multiple online model selection baselines.

\textbf{Latent bandits.}
Our setting is also an instance of latent bandits, where bandit instances are parameterized by a finite set of latent states, and each one corresponds to a different hypothesis over reward models \citep{latent_bandits,latent_contextual_bandits,latent_bandits_revisited}. In such structured environments, it is natural to consider a mixture prior that is learned from existing data, each component of the prior being a model distribution. However, most previous works only considered a single fixed model per latent state \citep{latent_bandits}. The closest work is \citet{latent_bandits_revisited}, who proposed a TS algorithm \mmts. In a bandit, \mts is an instance of \mmts where the conditional models are mixture components and share the same parameter space. This distinction is important, as we explicitly analyze the concentration of the mixture posterior to derive sublinear regret bounds. The regret bounds of \citet{latent_bandits_revisited} are agnostic to posterior improvements and can be linear. We also apply \mts to reinforcement learning, which in turn generalizes \mmts.

\section{CONCLUSIONS}
\label{sec:conclusions}

We propose Thompson sampling with a mixture prior (\mts) for online decision making. The mixture prior is parameterized by a discrete latent state, and yields a general and tractable algorithm that can be broadly analyzed, in both bandit and RL settings. Our regret bounds reflect the structure of the prior, the number of mixture components and their widths. We evaluate \mts on both synthetic and an image classification problems, and demonstrate that it performs well.

This work is a step towards analyzing TS in realistic models with latent variables. Our regret bounds depend on the number and width of the prior mixture components, but not on the latent state prior, which leaves room for improvement. We also only consider a flat discrete latent state. More expressive latent structures are an interesting direction for future work.

\bibliographystyle{abbrvnat}
\bibliography{paper}

\appendix
\onecolumn

\section{Linear Bandit Proofs}
\label{sec:lin_bandit_proofs}
\subsection{Useful Lemmas}

\begin{lemma}
\label{lem:gaussian_norm_tail_bound} Let $X \in \mathbb{R}^d$ be a random vector sampled from the multivariate Gaussian $X \sim \mathcal{N}(0, \Sigma)$. For any $\varepsilon \geq 0$, define event $E = \left\{ \norm{X}{\Sigma^{-1}} \geq \varepsilon \right\}$. Then,
\begin{align*}
    \E{}{\norm{X}{\Sigma^{-1}} \indicator{E}} 
    \leq 
    \frac{1}{\sqrt{2\pi}} d^{3/2} \exp\left(-\frac{\varepsilon^2}{2d}\right) 
\end{align*}
\end{lemma}
\begin{proof}
Using that $X \sim \mathcal{N}(0, \Sigma)$, we can conclude that 
$
\Sigma^{-1/2}X \sim \mathcal{N}(0, I_d)
$
has independent Gaussian entries. We have 
\begin{align*}
    \E{}{\norm{X}{\Sigma^{-1}}\indicator{E}}
    = 
    \E{}{\norm{\Sigma^{-1/2}X}{2} \indicator{E}}
    &\leq 
    \sqrt{d} \, \E{}{\norm{\Sigma^{-1/2}X}{\infty} \indicator{E}} \\
    &\leq 
    \sqrt{d} \sum_{i = 1}^d \frac{1}{\sqrt{2\pi}} \int_{u = \varepsilon / \sqrt{d}}^{\infty} u \exp\left(-\frac{u^2}{2}\right) \dif u \\
    &= \sqrt{d} \sum_{i = 1}^d -\frac{1}{\sqrt{2\pi}} \int_{u = \varepsilon / \sqrt{d}}^{\infty} \left(\exp\left(-\frac{u^2}{2}\right)\right)' \dif u \\
    &= 
    \frac{1}{\sqrt{2\pi}} d^{3/2} \exp\left(-\frac{\varepsilon^2}{2d}\right)\,,
\end{align*} 
where we use that $\norm{\Sigma^{-1/2}X}{2} \geq \varepsilon$ implies $\norm{\Sigma^{-1/2}X}{\infty} \geq \varepsilon/\sqrt{d}$, and consider each entry of $\Sigma^{-1/2}X$ separately.
\end{proof}

\begin{lemma}
For round $t$ and latent state $s$, let $\theta_{t, s}, \Sigma_{t, s}$ be defined as in \eqref{eq:mixture_posterior_linbandit}. If $\norm{A_t}{2} \leq \kappa$ for all $t$, then for any $C > 0$ such that $\lambda_{\max}(\Sigma_{0, s}) \leq \sigma^2C / \kappa^2$, then
\begin{align*}
    \sum_{t = 1}^n \norm{A_t}{\Sigma_{t, s}}^2
    \leq \sigma^2 (1 + C) \log \frac{\determinant{\Sigma_{n+1, s}^{-1}}}{\determinant{\Sigma^{-1}_{0, s}}}.
\end{align*}
\label{lem:norm_series_bound}
\end{lemma}

\begin{proof}
The proof is similar to that done for Lemma 11 of \citet{linucb}. Instead, we consider the norm with respect to posterior covariance $\Sigma_{t, s}$ rather than empirical covariance $V_t^{-1}$.

We have,
\begin{align*}
    \determinant{\Sigma^{-1}_{n+1, s}} = \determinant{\Sigma^{-1}_{n, s} + \sigma^{-2} A_{n}A_{n}^\top}
    &= \determinant{\Sigma^{-1}_{n, s}}\left(1 + \norm{\sigma^{-2}A_{n}}{\Sigma_{n, s}}^2\right) \\
    &= \determinant{\Sigma_{0, s}^{-1}}\prod_{t = 1}^n\left(1 + \sigma^{-2}\norm{A_t}{\Sigma_{t, s}}^2\right),
\end{align*}
where we use the matrix determinant lemma, which says $\determinant{A + uu^\top} = \determinant{A}\left(1 + \norm{u}{A^{-1}}^2\right)$ for matrix $A$ and vector $u$. 
Note that
\begin{align*}
    \norm{A_t}{\Sigma_{t, s}}^2 \leq \lambda_{\max}(\Sigma_{t, s}) \norm{A_t}{2}^2 \leq \kappa^2 \lambda_{\max}(\Sigma_{0, s}) \,,
\end{align*}
so if $\lambda_{\max}(\Sigma_{0, s}) \leq \sigma^2C/\kappa^2$, then $\sigma^{-2} \norm{A_t}{\Sigma_{t, s}}^2 \leq C$. 
Using that $x \leq (1 + C)\log(1 + x)$ for $x \in [0, C]$, we get,
\begin{align*}
    \sum_{t = 1}^n \sigma^{-2}\norm{A_{t, s}}{\Sigma_{t, s}}^2 
    &\leq 
    (1 + C) \sum_{t = 1}^n \log\left(1 + \sigma^{-2}\norm{A_{t, s}}{\Sigma_{t, s}}^2\right) \\
    &\leq 
    (1 + C) \log \frac{\determinant{\Sigma_{n+1, s}^{-1}}}{\determinant{\Sigma^{-1}_{0, s}}}\,.
\end{align*}
This yields
\begin{align*}
    \sum_{t = 1}^n \norm{A_t}{\Sigma_{t, s}}^2 \leq \sigma^2(1 + C) \log \frac{\determinant{\Sigma_{n+1, s}^{-1}}}{\determinant{\Sigma^{-1}_{0, s}}}\,,
\end{align*}
as desired.
\end{proof}

\subsection{Proof of \cref{thm:bayes_regret_linbandit_mixture}}
In the outline, we were able to trivially bound the regret of each round by $1$; this is no longer the case since $\theta_*$ is a sample from a Gaussian.
To handle this, we introduce event
\begin{align*}
    E_0 = \left\{\norm{\theta_* - \theta_{0, S_*}}{\Sigma_{0, S_*}^{-1}} \leq \sqrt{2d\log(dn)}\right\},
\end{align*}
which occurs when $\theta_*$ is not far from its prior mean. We can bound the regret by,
\begin{align*}
\E{}{\sum_{t = 1}^n A_{t, *}^\top \theta_* - A_t^\top\theta_*} 
\leq 
\E{}{\sum_{t = 1}^n (A_{t, *}^\top \theta_* - A_t^\top\theta_*)\indicator{E_0}} + \E{}{\sum_{t = 1}^n (A_{t, *}^\top \theta_* - A_t^\top\theta_*)\indicator{\bar{E}_0}}\,.
\end{align*}

When $\bar{E}_0$ occurs, the regret for a round can be arbitrarily large; to handle this, we factor in that $\bar{E}_0$ is unlikely.
Fix round $t$. We bound the regret in round $t$ as
\begin{align*}
    \E{t}{(A_{t, *}^\top \theta_* - A_t^\top\theta_*)\indicator{\bar{E}_0}} 
    &\leq
    \E{t}{A_{t, *}^\top (\theta_* - \theta_{0, S_*})\indicator{\bar{E}_0}} + 
    \E{t}{A_{t, *}^\top \theta_{0, S_*}\indicator{\bar{E}_0}} \\
    &\leq 
    \E{t}{\norm{A_{t, *}}{\Sigma_{0, S_*}}\norm{\theta_* - \theta_{0, S_*}}{\Sigma_{0, S_*}^{-1}}\indicator{\bar{E}_0}} + 
    \E{t}{\norm{A_{t, *}}{2}\norm{\theta_{0, S_*}}{2}\indicator{\bar{E}_0}} \\
    &\leq 
    \sqrt{\kappa^2\lambda_{0, \max}} \, \E{}{\norm{\theta_* - \theta_{0, S_*}}{\Sigma_{0, S_*}^{-1}}\indicator{\bar{E}_0}}
    + 
    \kappa\prob{}{\bar{E}_0} \,,
\end{align*}
where we use the Cauchy-Schwartz inequality, and $\norm{a}{\Sigma_{0, s}} \leq \sqrt{\lambda_{\max}(\Sigma_{0, s})}\norm{a}{2} \leq \sqrt{\kappa^2\lambda_{0, \max}}$ for any action $a$ and latent state $s$.
Since $\theta_* - \theta_{0, S_*} \sim \mathcal{N}(0, \Sigma_{0, S_*})$, we have $\prob{}{\bar{E}_0} \leq n^{-1}$ and
\begin{align*}
    \E{}{\norm{\theta_* - \theta_{0, S_*}}{\Sigma_{0, S_*}^{-1}} \indicator{\bar{E}_0}}
    \leq
    \sqrt{\frac{d}{2\pi}} n^{-1}\,,
\end{align*} 
where we apply \cref{lem:gaussian_norm_tail_bound} with $\varepsilon = \sqrt{2d\log(dn)}$.
Hence, we can bound the Bayes regret as
\begin{align*}
\E{}{\sum_{t = 1}^n A_{t, *}^\top \theta_* - A_t^\top\theta_*} 
\leq 
\E{}{\sum_{t = 1}^n (A_{t, *}^\top \theta_* - A_t^\top\theta_*)\indicator{E_0}} + \sqrt{\frac{\kappa^2\lambda_{0,\max}d}{2\pi}} + \kappa\,.
\end{align*}
When $E_0$ occurs, we have $M = \sqrt{2\kappa^2\lambda_{0, \max}d\log(dn)} + \kappa$ is an upper-bound on regret for a round. We use $\langle \cdot \rangle_M = \min\{\cdot, M\}$.

From here, we can follow the analysis outline in \cref{sec:proof_sketch} using $\bar{\mu}_t, \sigma_t$ defined as
\begin{align*}
    \bar{\mu}_t(a, s) = a^\top\bar{\theta}_{s, t}, \
    \sigma_t(a, s) = \norm{a}{\Sigma_{t, s}}\sqrt{2d\log(dn)}\,.
\end{align*}
Using \cref{eq:bayes_regret_decomposition}, we can decompose the Bayes regret in a linear bandit as
\begin{align}
&\E{}{\sum_{t = 1}^n (A_{t, *}^\top \theta_* - A_t^\top\theta_*)\indicator{E_0}} \nonumber \\
&\,\leq 
\E{}{\sum_{t = 1}^n\E{t}{(A_{t, *}^\top \theta_* - A_{t,*}^\top\bar{\theta}_{t, S_*})\indicator{E_0}}} 
+ 
\E{}{\sum_{t = 1}^n\E{t}{(A_t^\top\bar{\theta}_{t, S_t} - A_t^\top\theta_*)\indicator{E_0}}} \,.
\label{eq:bayes_regret_decomposition_linbandit}
\end{align}
We bound each term individually. 

\paragraph{Step 1.}
Let us first consider the first term of \eqref{eq:bayes_regret_decomposition_linbandit}. Fix round $t$. Let us define event
\begin{align*}
    E_t = \left\{\norm{\theta_* - \bar{\theta}_{t, S_*}}{\Sigma_{t, S_*}^{-1}} \leq \sqrt{2d\log(dn)}\right\},
\end{align*} 
which occurs when $\theta_*$ is not far from the mean of conditional posterior $P_t(\cdot \mid S_*) = \mathcal{N}(\bar{\theta}_{t, S_*}, \Sigma_{t, S*})$. Note that $E_1 = E_0$ from earlier.
We can bound
\begin{align*}
    \E{t}{A_{t, *}^\top\theta_* - A_{t,*}^\top\bar{\theta}_{t, S_*}} 
    &=  
    \E{t}{\left(A_{t, *}^\top\theta_* - A_{t, *}^\top\bar{\theta}_{t, S_*}\right) \indicator{E_t}} 
    +
    \E{t}{\left(A_{t, *}^\top\theta_* -  A_{t, *}^\top\bar{\theta}_{t, S_*})\right) \indicator{\bar{E}_t}}
    \\
    &\leq 
    \E{t}{\norm{A_{t, *}}{\Sigma_{t, S_*}}\sqrt{2d\log(dn)}} +
    \E{t}{\left(A_{t, *}^\top\theta_* - A_{t, *}^\top\bar{\theta}_{t, S_*}\right) \indicator{\bar{E}_t}}\,,
\end{align*}
where we use that when $E_t$ occurs, we have
\begin{align*}
    (A_{t, *}^\top\theta_* - A_{t, *}^\top\bar{\theta}_{t, S_*})\indicator{E_t} 
    \leq \norm{A_{t, *}}{\Sigma_{t, S_*}}\norm{\theta_* - \bar{\theta}_{t, s}}{\Sigma_{t, S_*}^{-1}} \indicator{E_t}
    \leq \norm{A_{t, *}}{\Sigma_{t, S_*}}\sqrt{2d\log(dn)}\,.
\end{align*}

Now, when $\bar{E}_t$ occurs, we have
\begin{align*}
    \E{t}{\left(A_{t, *}^\top\theta_* - A_{t, *}^\top\bar{\theta}_{t, S_*}\right) \indicator{\bar{E}_t}}
    &\leq 
    \kappa \sqrt{\lambda_{0, \max}} \, \E{t}{\norm{\theta_* - \bar{\theta}_{t, S_*}}{\Sigma_{t, S_*}^{-1}} \indicator{\bar{E}_t}}\,,
\end{align*}
where we again use Cauchy-Schwartz and $\norm{a}{\Sigma_{t, s}} \leq \sqrt{\kappa^2\lambda_{0, \max}}$ for any action $a$ and latent state $s$.

Using that that $\theta_* - \bar{\theta}_{t, s} \mid H_t \sim \mathcal{N}(0, \Sigma_{t, s})$,
we can use \cref{lem:gaussian_norm_tail_bound} with $\varepsilon = \sqrt{2d\log(dn)}$ to bound
\begin{align*}
    \E{t}{\norm{\theta_* - \bar{\theta}_{t, S_*}}{\Sigma_{t, S_*}^{-1}} \indicator{\bar{E}_t}}
    &\leq \sqrt{\frac{d}{2\pi}} n^{-1} \,.
\end{align*} 
Hence, we can bound the first term of \eqref{eq:bayes_regret_decomposition_linbandit} by
\begin{align*}
\E{}{\sum_{t = 1}^n \E{t}{A_{t, *}^\top \theta_* -A_{t, *}^\top\bar{\theta}_{t, S_*}}}
\leq 
\sqrt{2d\log(dn)} \, \E{}{\sum_{t=1}^n\norm{A_{t, *}}{\Sigma_{t, S_*}}} + \sqrt{\frac{\kappa^2\lambda_{0, \max} d}{2\pi}}\,.
\end{align*}

\paragraph{Step 2.}
We define $C_t$ as a high-probability set around latent states using the following construction:
\begin{align*}
    C_t = \left\{s \in \Sset: G_t(s) \leq 2\sigma\sqrt{N_t(s) \log{n}}\right\}\,,
\end{align*}
where $N_t(s) = \sum_{\ell = 1}^{t-1} \indicator{S_t = s}$ and
\begin{align*}
    G_t(s) = \sum_{\ell = 1}^{t - 1} \indicator{S_t = s} \left(A_t^\top \bar{\theta}_{t, s} - \norm{A_t}{\Sigma_{t, s}} \sqrt{2d\log{n}} - Y_t\right)
\end{align*}
is the \say{over-estimation} of the predicted rewards under a latent state and the realized reward.
We show that $S_* \in C_t$ holds with high probability for any round via the following lemma.

\begin{lemma}
\label{lem:gap_concentration}
For any round $t$, $\prob{}{S_* \not\in C_t} \leq 2Ln^{-1}$.
\vspace{-0.05in}
\end{lemma}
\begin{proof}
We know that $S_* \in C_t$ occurs if $G_t(S_*)$ is not too large. On a high-level, our goal is to upper-bound $G_t(S_*)$ by a martingale with respect to history, then bound the probability that $G_t(S_*)$ is too large using Azuma's inequality for concentration of martingales.

For $\ell < t$, we know that $\theta_* - \bar{\theta}_{\ell, s} \mid H_\ell \sim \mathcal{N}(0, \Sigma_{\ell, s})$. Let us define 
\begin{align*}
\mathcal{E}_\ell
= 
\left\{\norm{\theta_* - \bar{\theta}_{\ell, S_*}}{\Sigma_{\ell, S_*}^{-1}} \leq 2\sqrt{d\log{n}},\, \right\}
\end{align*}
as the event that $\theta_*$ is not too far from its posterior mean.
Let $\mathcal{E} = \cap_{\ell = 1}^{t-1}\{\mathcal{E}_\ell\}$ be the event that this holds for all rounds up to round $t$ and $\bar{\mathcal{E}}$ be the complement. We know that
\begin{align*}
    \indicator{S_* \not\in C_t} 
    =
    \indicator{G_t(S_*) \geq 2\sigma\sqrt{N_t(S_*) \log n}} 
    \leq 
    \indicator{\mathcal{\bar{E}}} + \indicator{\mathcal{E}}\indicator{G_t(S_*) \geq 2\sigma\sqrt{N_t(S_*) \log n}}\,,
\end{align*}
which implies that
\begin{align}
    \prob{}{S_* \not\in C_t} \leq \prob{}{\bar{\mathcal{E}}} + \prob{}{G_t(S_*)\indicator{\mathcal{E}} \geq  2\sigma \sqrt{N_t(S_*) \log n}}\,.
\label{eq:gap_probability_decomposition}
\end{align}
We will bound each probability individually.
For the first probability of \eqref{eq:gap_probability_decomposition}, we simply have
\begin{align*}
\prob{}{\bar{\mathcal{E}}} 
\leq 
\sum_{s \in \Sset}\sum_{\ell = 1}^{t-1}\E{}{\prob{\ell}{\norm{\theta_* - \bar{\theta}_{\ell, s}}{\Sigma_{\ell, s}^{-1}} \geq 2\sqrt{d\log{n}}}}
\leq Ln^{-1}\,,
\end{align*} 
where we use that for $S_* = s$ and round $\ell$, we have $\norm{\theta_* - \bar{\theta}_{\ell, s}}{\Sigma_{\ell, s}^{-1}} \mid H_\ell$ is the sum of independent Gaussians. Then, we take an expectation over histories, and use a union bound over latent states and rounds.

Now, consider the second probability in \eqref{eq:gap_probability_decomposition}. Fix $S_* = s$, and let $\mathcal{T}_{t, s} = \{\ell < t: S_\ell = s\}$ be the rounds where $s$ is sampled up to round $t$. Also, let $Z_\ell = A_\ell^\top \theta_* - Y_\ell$. Observe that $Z_\ell \sim \mathcal{N}(0, \sigma^2)$, so that $(Z_\ell)_{t \in \mathcal{T}_{t, s}}$ is a martingale difference sequence with respect to histories $(H_\ell)_{t \in \mathcal{T}_{t, s}}$.
We have
\begin{align*}
&(A_\ell^\top \bar{\theta}_{t, \ell} - \norm{A_\ell}{\Sigma_{t, \ell}} \sqrt{2d\log{n}} - Y_\ell)\indicator{\mathcal{E}_\ell} \\
&\,=(A_\ell^\top \theta_* + A_\ell^\top (\bar{\theta}_{t, \ell} - \theta_*) - \norm{A_\ell}{\Sigma_{t, \ell}} \sqrt{2d\log{n}} - Y_\ell)\indicator{\mathcal{E}_\ell} \\
&\,\leq
(A_\ell^\top \theta_* +  \norm{A_\ell}{\Sigma_{t, \ell}}\norm{\bar{\theta}_{t, \ell} - \theta_*}{\Sigma_{t, \ell}^{-1}} - 2\norm{A_\ell}{\Sigma_{t, \ell}} \sqrt{d\log{n}} - Y_\ell)\indicator{\mathcal{E}_\ell}
\leq Z_\ell\,,
\end{align*}
where we use Cauchy-Schwartz in the inequality.
This implies that
\begin{align*}
    G_t(s)\indicator{\mathcal{E}} 
    =
    \sum_{\ell \in \mathcal{T}_{t, s}} (A_\ell^\top \bar{\theta}_{t, \ell} - 2\norm{A_\ell}{\Sigma_{t, \ell}} \sqrt{d\log{n}} - Y_\ell)\indicator{\mathcal{E}_\ell}
    \leq 
    \sum_{\ell \in \mathcal{T}_{t, s}} Z_\ell\,.
\end{align*}
For any round $t$, and latent state $s$, we have that $\mathcal{T}_{t, s}$ is a random quantity. First, we fix $|\mathcal{T}_{t, s}| = N_t(s) = u$ where $u < t$ and yield the following due to Azuma's inequality, 
\begin{align*}
  \prob{}{G_t(s)\indicator{\mathcal{E}} \geq  2\sigma \sqrt{u \log n}} 
  &\leq 
  \prob{}{\sum_{\ell \in \mathcal{T}_{t, s}} Z_\ell(s) \geq  2\sigma \sqrt{ u \log n}}
  \leq \exp\left[-2\log n\right]
  = n^{-2}\,.
\end{align*}
Finally, by the union bound, we have
\begin{align*}
  \prob{}{G_t(S_*)\indicator{\mathcal{E}} \geq  2\sigma \sqrt{N_t(S_*) \log n}} 
  &\leq 
  \sum_{s \in \Sset} \sum_{u = 1}^{t-1}
  \prob{}{G_t(s)\indicator{\mathcal{E}} \geq  2\sigma \sqrt{u \log n}}
  \leq L n^{-1}\,.
\end{align*}
Combining the two bounds completes the proof.
\end{proof}

\paragraph{Step 3.}
Now, we consider the second term of \eqref{eq:bayes_regret_decomposition_linbandit}. We have,
\begin{align*}
    \E{}{\sum_{t = 1}^n \langle A_t^\top\bar{\theta}_{t, S_t} - A_t^\top\theta_* \rangle_M}
    &\leq 
    M \sum_{t = 1}^n \prob{}{S_t \not\in C_t} + 
    \E{}{\sum_{t = 1}^n \langle A_t^\top\bar{\theta}_{t, S_t} - A_t^\top\theta_*\rangle_M \indicator{S_t \in C_t}}
     \\
    &\leq 
    M\sum_{t = 1}^n \prob{}{S_* \not\in C_t} +
    \E{}{\sum_{t = 1}^n \langle A_t^\top\bar{\theta}_{t, S_t} - A_t^\top\theta_*\rangle_M \indicator{S_t \in C_t}}
\end{align*}
where we use that conditioned on $H_t$, $S_t, S_*$ are i.i.d. to get $\prob{}{S_t \in C_t} = \E{}{\prob{t}{S_t \in C_t}} = \E{}{\prob{t}{S_* \in C_t}} = \prob{}{S_* \in C_t}$ From \cref{lem:gap_concentration}, the first term is $2LM$. From the outline in \cref{sec:proof_sketch}, we have
\begin{align*}
    &\E{}{\sum_{t = 1}^n \langle A_t^\top\bar{\theta}_{t, S_t} - A_t^\top\theta_*\rangle_M \indicator{S_t \in C_t}} \\
    &\,\leq
    2\sqrt{d\log(dn)}\,\E{}{\sum_{t = 1}^n \norm{A_t}{\Sigma_{t, S_t}}}
    + \E{}{\sum_{t = 1}^n \langle A_t^\top\bar{\theta}_{t, S_t} - 2\norm{A_t}{\Sigma_{t, S_t}}\sqrt{d\log(dn)} - Y_t \rangle_M \indicator{S_t \in C_t}}\,.
\end{align*}
The last term can be bounded as
\begin{align*}
    \sum_{t = 1}^n \langle A_t^\top\bar{\theta}_{t, S_t} - 2\norm{A_t}{\Sigma_{t, S_t}}\sqrt{d\log(dn)} - Y_t \rangle_M \indicator{S_t \in C_t}
    \leq 
    \sum_{s \in \Sset} G_{n}(s) + LM
    \leq 
    2\sigma\sqrt{Ln \log n} + LM \,,
\end{align*}
where for latent state $s$, and $t' = \max_{t \in [n]} \{S_t = s\}$ as the last round that a latent state $s$ is acted upon, we use that there is an upper-bound on $G_{t'}(s) \leq G_{n}(s)$ by definition of $s \in C_t$. We trivially bound the regret by $M$ for the last round $s$ is acted upon.

Hence, we can bound the second term of \eqref{eq:bayes_regret_decomposition_linbandit} by
\begin{align*}
\E{}{\sum_{t = 1}^n \E{t}{(A_t^\top\bar{\theta}_{t, S_t} - A_t^\top \theta_*)\indicator{E_0}}}
\leq
2\sqrt{d\log{n}}\,\E{}{\sum_{t = 1}^n \norm{A_t}{\Sigma_{t, S_t}}} +
2\sigma\sqrt{Ln \log n} + 3LM \,.
\end{align*}

What remains is bounding the sum of confidence widths. We have
\begin{align*}
    \sum_{t = 1}^n \norm{A_t}{\Sigma_{t, S_t}} \leq
    \sum_{t = 1}^n \max_{s \in \Sset} \norm{A_t}{\Sigma_{t, s}}
    &\leq 
    \max_{s \in \Sset} \sqrt{n\sum_{t = 1}^n\norm{A_t}{\Sigma_{t, s}}^2} \\
    &\leq 
    \max_{s \in \Sset} \sqrt{\sigma^2\left(1 + \frac{\kappa^2\lambda_{0, \max}}{\sigma^2}\right)n\log\left(\frac{\determinant{\Sigma_{n + 1, s}^{-1}}}{\determinant{\Sigma_{0, s}^{-1}}}\right)} \\
    &\leq 
    \sqrt{\sigma^2\left(1 + \frac{\kappa^2\lambda_{0, \max}}{\sigma^2}\right)nd\log\left(1 + n\frac{\kappa^2 \lambda_{0, \max}}{\sigma^2d} \right)}\,,
\end{align*} 
where we first use that $\norm{A_t}{\Sigma_{t, s}}$ for latent states $s$ differ among one another only through their prior, and then use \cref{lem:norm_series_bound} to bound the sum of norms. We use the determinant-trace inequality to bound,
\begin{align*}
    \log\left(\frac{\determinant{\Sigma_{n + 1, s}^{-1}}}{\determinant{\Sigma_{0, s}^{-1}}}\right) 
    \leq
    d \log\left(\frac{\trace{\Sigma_{0, s}^{-1}} + n\sigma^{-2}\kappa^2}{\trace{\Sigma_{0, s}^{-1}}}\right)
    \leq 
    d \log\left(1 + n\frac{\kappa^2\lambda_{0, \max}}{\sigma^2d}\right)\,,
\end{align*}
where we use that
\begin{align*}
    \trace{\Sigma_{0, s}^{-1}} \geq \lambda_{\min}(\Sigma_{0, s}^{-1})d =
    \lambda_{\max}^{-1}(\Sigma_{0, s})d \geq
    \lambda_{0, \max}^{-1}d\,.
\end{align*}
Combining the bounds across all steps yields
\begin{align*}
    \Bregret(n) 
    &\leq
    4d\sqrt{\sigma^2\left(1 + \frac{\kappa^2\lambda_{0, \max}}{\sigma^2}\right)n\log(dn)\log\left(1 + n\frac{\kappa^2 \lambda_{0, \max}}{\sigma^2d} \right)} + 2\sqrt{\sigma^2Ln \log n} \\
    &\qquad + 3L\sqrt{2\kappa^2\lambda_{0, \max}d\log(dn)} + 2\sqrt{\frac{\kappa^2\lambda_{0,\max}d}{2\pi}} + 4L\kappa \,.
\end{align*}
\qed

\clearpage

\section{Tabular MDP Proofs}
\label{sec:mdp_proofs}
\subsection{Useful Lemmas}
\begin{lemma}[Theorem 1 and 3 of \citet{dirichlet_concentration}]
\label{lem:beta_dir_sub-gaussianity} Let $X \sim \mathrm{Beta}(\alpha, \beta)$ for $\alpha, \beta > 0$. Then $X - \E{}{X}$ is $\sigma^2$-sub-Gaussian with
$
  \sigma^2
  = 1/(4 (\alpha + \beta + 1))
$.
Similarly, let $X \sim \mathrm{Dir}(\alpha)$ for $\alpha \in \mathbb{R}_{+}^d$. Then $X - \E{}{X}$ is $\sigma^2$-sub-Gaussian with
$
  \sigma^2
  = 1/(4 (\norm{\alpha}{1} + 1))
$.
\end{lemma}

\begin{lemma}[Value difference lemma]
For any MDPs $M'$, $M$, and policy $\pi$,
\begin{align*}
V_{M'}(\pi) - V_M(\pi)
\leq 
\mathbb{E}\left[\sum_{i = 1}^{h} R_{M'}(X_i, A_i) - R_M(X_i, A_i)
h\norm{T_{M'}(X_i, A_i) - T_M(X_i, A_i)}{1}\right]\,.
\end{align*}
\label{lem:mdp_value_difference}
\vspace{-0.05in}
\end{lemma}

\begin{lemma}
For episode $t$ and state $s$, let $\beta_t(s, x, a) = c_t(s, x, a) + \phi_t(s, x, a)$ for any $(x, a)$ as in \eqref{eq:mixture_posterior_tabular_mdp_r}, \eqref{eq:mixture_posterior_tabular_mdp_p}, respectively. Let $\Lambda_{0, s} = \min\{\min_{x, a}\norm{\alpha_{0, s}^R(x, a)}{1}, \min_{x, a}\norm{\alpha_{0, s}^T(x, a)}{1}\}$ represent at least how concentrated the reward and transition priors are for latent state $s$, where higher values correspond to lower prior widths. Then we have that
\begin{align*}
h\sum_{t = 1}^n \sum_{i = 1}^h \beta_t(X_{t, i}, A_{t, i}, s) 
\leq
    4|\Xset|h\sqrt{|\Aset|nh\log(4|\Xset||\Aset|n) \log\left(1 + \frac{nh}{2 |\Xset||\Aset|\Lambda_{0, s}}\right)}
    + |\Xset||\Aset|h^2\,.
\end{align*}
\label{lem:tabular_mdp_sum}
\vspace{-0.1in}
\end{lemma}

\begin{proof}
The proof is similar to that done in \citet{mdp_posterior_sampling}. However, we use prior-dependent definitions for the confidence width $\beta_t$. First, we define
$
N_t(x, a) = \sum_{\ell = 1}^{t-1}\sum_{i = 1}^h \indicator{X_{\ell, i} = x, A_{\ell, i} = a}
$
as the number of times $x, a$ were sampled up to episode $t$.
We can decompose the sum as
\begin{align*}
    \sum_{t = 1}^n \sum_{i = 1}^h \beta_t(X_{t, i}, A_{t, i}, s) \leq 
    \sum_{t = 1}^n \sum_{i = 1}^h \indicator{N_t(X_{t, i}, A_{t, i}) \leq h} + 
     \sum_{t = 1}^n \sum_{i = 1}^h \indicator{N_t(X_{t, i}, A_{t, i}) > h} \beta_t(X_{t, i}, A_{t, i}, s)\,,
\end{align*}
where we trivially bound the regret in a step of an episode by $1$. Therefore, the first term is bounded as $|\Xset||\Aset|h$. 

For the second term, let us additionally define $N_{t, i}(x, a) = N_t(x, a) + \sum_{k = 1}^{i-1} \indicator{X_{t, k} = x, A_{t, k} = a}$ as the number of times $x, a$ were sampled up to step $i$ of episode $t$. Now, if $N_t(x, a) > h$, then we know that
$
    N_{t, i}(x, a)
    \leq N_t(x, a) + h
    \leq 2N_t(x, a)
$. We consider $c_t, \phi_t$ of $\beta_t$ individually. We have 
\begin{align*}
    &\sum_{t = 1}^n \sum_{i = 1}^h \indicator{N_t(X_{t, i}, A_{t, i}) > h}c_t(X_{t, i}, A_{t, i}, s) \\
    &\,= 
    \sum_{t = 1}^n \sum_{i = 1}^h \indicator{N_t(X_{t, i}, A_{t, i}) > h} \sqrt{\frac{2\log(2|\Xset||\Aset|n)}{\norm{\alpha^R_{t, s}(x, a)}{1} + 1}} \\
    &\,= 
    \sum_{x, a} \sum_{t = 1}^n \sum_{i = 1}^h \indicator{N_t(x, a) > h} \sqrt{\frac{2\log(2|\Xset||\Aset|n)}{\norm{\alpha^R_{0, s}(x, a)}{1} + N_t(x, a) + 1}} \\
    &\,\leq 
    \sum_{x, a} \sum_{t = 1}^n \sum_{i = 1}^h \sqrt{\frac{4\log(2|\Xset||\Aset|n)}{2\norm{\alpha^R_{0, s}(x, a)}{1} + N_{t, i}(x, a)}} \\
    &\,\leq 
    2\sqrt{\log(2|\Xset||\Aset|n)} \sum_{x, a} \sqrt{N_{n + 1}(x, a) \sum_{u = 1}^{N_{n + 1}(x, a)} \frac{1}{2\norm{\alpha^R_{0, s}(x, a)}{1} + u}} \\
    &\,\leq 
    2\sqrt{|\Xset||\Aset|nh\log(2|\Xset||\Aset|n)} \sqrt{\sum_{u = 1}^{nh/|\Xset||\Aset|} \frac{1}{2\Lambda_{0, s} + u}} \\ \\
    &\,\leq 
    2\sqrt{|\Xset||\Aset|nh\log(2|\Xset||\Aset|n) \log\left(1 + \frac{nh}{2 |\Xset||\Aset|\Lambda_{0, s}}\right)} \,, \\
\end{align*}
where for the last inequality, we use that for any $x > 0$,
\begin{align*}
    \sum_{u = 1}^{nh/|\Xset||\Aset|} \frac{1}{x + u}
    \leq
    \int_{u = x}^{x + nh/|\Xset||\Aset|} u^{-1}du
    \leq
    \log\left(1 + \frac{nh}{|\Xset||\Aset|x}\right)\,.
\end{align*}
Similarly, we have
\begin{align*}
    &\sum_{t = 1}^n \sum_{i = 1}^h \indicator{N_t(X_{t, i}, A_{t, i}) > h}\phi_t(s, X_{t, i}, A_{t, i}) \\
    &\,\leq 
    2|\Xset|h\sqrt{2|\Aset|nh\log(4|\Xset||\Aset|n) \log\left(1 + \frac{nh}{2 |\Xset||\Aset|\Lambda_{0, s}}\right)} \,. \\
\end{align*}
Combining the two bounds yields
\begin{align*}
    h\sum_{t = 1}^n \sum_{i = 1}^h \beta_t(s, X_{t, i}, A_{t, i})
    \leq 
    4|\Xset|h\sqrt{2|\Aset|nh\log(4|\Xset||\Aset|n) \log\left(1 + \frac{nh}{2 |\Xset||\Aset|\Lambda_{0, s}}\right)}
    + |\Xset||\Aset|h^2 \,.
\end{align*}
\end{proof}

\subsection{General Analysis Outline}
\label{sec:proof_sketch_mdp}
\paragraph{Step 1.}
Bound the Bayes regret due to the first term of \eqref{eq:bayes_regret_decomposition_mdp}.
For episode $t$, we introduce event
\begin{align*}
    E_t = \left\{\forall(x, a): \abs{R_{M_t}(x, a) - \bar{r}_t(x, a, S_t)} \leq c_t(x, a, S_t) \,,\, \norm{T_{M_t}(x, a) - \bar{p}_t(x, a, S_t)}{1} \leq \phi_t(x, a, S_t) \right\}
\end{align*}
to denote when the sampled mean rewards and transitions are not far from their posterior means for all state-action pairs. Using \cref{lem:mdp_value_difference}, we know that
\begin{align*}
    &\E{t}{V_*(\pi_*) - \overline{V}_t(\pi_*, S_*)} \\
    &\,= \E{t}{\E{M \sim P_t(\cdot \mid S_*)}{V_*(\pi_*) - V_M(\pi_*)}} \\
    &\,\leq \E{t}{\sum_{i = 1}^h (R_*(X_{t, i}, A_{t, i}) - r_t(X_{t, i}, A_{t, i}, S_*)) + h\norm{T_*(X_{t, i}, A_{t, i}) - \bar{p}_t(X_{t, i}, A_{t, i}, S_*)}{1}} \\
    &\,\leq \E{t}{h\sum_{i = 1}^h (R_*(X_{t, i}, A_{t, i}) - r_t(X_{t, i}, A_{t, i}, S_*) + \norm{T_*(X_{t, i}, A_{t, i}) - \bar{p}_t(X_{t, i}, A_{t, i}, S_*)}{1})\indicator{\bar{E}_t}} \\
    &\qquad+ \E{t}{h \sum_{i = 1}^h \beta_t(X_{t, i}, A_{t, i}, S_t)}\,,
\end{align*}
where $\beta_t(x, a, s) = c_t(x, a, s) + \phi_t(x, a, s)$.
Here, we take an expectation over MDPs to apply \cref{lem:mdp_value_difference}, then condition on $E_t$ occurring.
The second term can be bounded as a sum of confidence widths, and the remaining term can be bounded by using that conditioned on $H_t$, $\bar{E}_t$ is unlikely.

\paragraph{Step 2.}
For each episode $t$, construct $C_t$ such that $S_* \in C_t$ with high probability. To do so, we define $N_t(s) = \sum_{\ell = 1}^{k-1} \indicator{S_\ell = s}$ as the number of times $s$ was acted upon and
\begin{align*}
    G_t(s) = \sum_{\ell=1}^{t - 1} \indicator{S_\ell = s} \left(\overline{V}_t(\pi_t, s) - \eta h\sum_{i = 1}^{h}\beta_\ell(X_{\ell, i}, A_{\ell, i}, s) - \sum_{i = 1}^{h} R_{\ell, i} \right)
\end{align*}
as the total over-estimation of observed returns by assuming that $s$ is the true latent state, where $\eta \in \mathbb{R}$ is a scaling factor. Here, we use the shorthand $\beta_t(x, a, s) = c_t(x, a, s) + \phi_t(x, a, s)$. Then we define $C_t$ as containing all latent states $s$ where $G_t(s) = \mathcal{O}(\sqrt{N_t(s)h\log{n}})$. Note that we scale by an additional $\mathcal{O}(\sqrt{h})$ over the outline for bandits in \cref{sec:proof_sketch} to account for taking the summation over a trajectory. We show that for any episode $t$, $\prob{}{S_* \not\in C_t} = \mathcal{O}(1/n)$. This means that with high probability, the true latent state lies in $C_t$.

\paragraph{Step 4.}
We can decompose the second term of \eqref{eq:bayes_regret_decomposition_mdp} as
\begin{align*}
\E{}{\sum_{t = 1}^n \overline{V}_t(\pi_t, S_t) -  V_*(\pi_t)} 
&\leq 
\E{}{\sum_{t = 1}^n \E{t}{(\overline{V}_t(\pi_t, S_t) -  V_*(\pi_t))\indicator{S_t \in C_t}}} 
+ h \sum_{t = 1}^n \prob{}{S_t \not\in C_t} \\
&\leq 
\E{}{\sum_{t = 1}^n \E{t}{(\overline{V}_t(\pi_t, S_t) -  V_*(\pi_t))\indicator{S_t \in C_t}}} 
+ h \sum_{t = 1}^n \prob{}{S_* \not\in C_t}\,,
\end{align*}
where we use that conditioned on $H_t$, latent states $S_*, S_t$ are identically distributed. From Step 1 and 2, we know that the second term is bounded as $2Lh$.
Finally, the remaining term can be bounded as
\begin{align*}
    &\E{}{\sum_{t = 1}^n (\overline{V}_t(\pi_t, S_t) -  V_*(\pi_t))\indicator{S_t \in C_t}} \\
    &\,= 
    \eta h\E{}{\sum_{t = 1}^n \sum_{i = 1}^{h} \beta_t(X_{t, i}, A_{t, i}, S_t)} +
    \E{}{\sum_{t = 1}^n \left(\overline{V}_t(\pi_t, S_t) - \eta h\sum_{i = 1}^h \beta_t(X_{t, i}, A_{t, i}, S_t) - \sum_{i=1}^h R_{t, i}\right)\indicator{S_t \in C_t}} \,,
\end{align*} 
where we use that $\E{t}{\sum_{i = 1}^h R_{t, i} \mid \pi_t, M_*} = \E{t}{V_*(\pi_t)}$.
The second can be bounded as the sum of confidence widths, 
which concentrate over time. The remaining term can be bounded by the sum of gaps $\sum_{s \in \Sset} G_{n+1}(s)$, which we know is bounded by $\mathcal{O}(\sqrt{Lnh\log{n}} + Lh)$ after trivially bounding the regret the last time each latent state is acted upon by $h$.

\subsection{Proof of \cref{thm:bayes_regret_mdp_mixture}}
Recall from the proof sketch in \cref{sec:proof_sketch_rl} that $\overline{V}_t(s, \pi) = \E{M \sim \bar{p}_t(\cdot \mid s)}{V_M(\pi)}$ is the expected value of a policy under state $s$, marginalized over MDPs sampled from its conditional posterior for episode $t$. We want to bound each term of the regret decomposition in \eqref{eq:bayes_regret_decomposition_mdp} separately.

\paragraph{Step 1.}
For episode $t$, let 
\begin{align*}
    E_t = \left\{\forall(x, a): \abs{R_{M_t}(x, a) - \bar{r}_t(x, a, S_t)} \leq c_t(x, a, S_t) \,,\, \norm{T_{M_t}(x, a) - \bar{p}_t(x, a, S_t)}{1} \leq \phi_t(x, a, S_t) \right\}
\end{align*}
denote the event that the sampled mean rewards and transitions are not far from their posterior means for all state-action pairs. From the sketch in \cref{sec:proof_sketch_mdp}, we rewrite the first term of \eqref{eq:bayes_regret_decomposition_mdp} as
\begin{align*}
    &\E{t}{V_{M_t}(\pi_t) - \overline{V}_t(\pi_t, S_t)} \\
    &\,\leq h \sum_{i = 1}^h \E{t}{(R_{M_t}(X_{t, i}, A_{t, i}) - \bar{r}_t(X_{t, i}, A_{t, i}, S_t) + \norm{T_{M_t}(X_{t, i}, A_{t, i}) - \bar{p}_t(X_{t, i}, A_{t, i}, S_t)}{1})\indicator{\bar{E}_t}} \\
    &\qquad+ h \sum_{i = 1}^h \E{t}{\beta_t(X_{t, i}, A_{t, i}, S_t)}\,.
\end{align*}
For each episode $t$, we can use that $R_{M_t}(x, a) \mid H_t \sim \mathrm{Beta}(\alpha^R_{t, S_t}(x, a))$ to yield
\begin{align*}
    \E{t}{(R_{M_t}(X_{t, i}, A_{t, i}) - \bar{r}_t(X_{t, i}, A_{t, i}, S_t))\indicator{\bar{E}_t}}
    &\leq \sum_{x, a} \int_{r = c_t(S_t, x, a)}^{\infty} r \prob{t}{R_{M_t}(x, a) - \bar{r}_t(x, a, S_t) = r} dr\\
    &\leq \sum_{x, a} \prob{t}{R_{M_t}(x, a) - \bar{r}_t(x, a, S_t) \geq  c_t(x, a, S_t)} \\
    &\leq \sum_{x, a}\exp\left[ - \frac{c_t(x, a, S_t)^2}{2/\left(4\left(\norm{\alpha^R_{t, S_t}(x, a)}{1} + 1\right)\right)}\right] \\
    &\leq 1/(2n)\,,
\end{align*}
where the second inequality uses that $R_{M_t}(x, a) \leq 1$ and the third uses the sub-Gaussian parameter given in \cref{lem:beta_dir_sub-gaussianity}. Similarly, since $T_{M_t}(x, a) \mid H_t \sim \mathrm{Dir}(\alpha^T_{k, S_t}(x, a))$, we have
\begin{align*}
    &\E{t}{\norm{T_{M_t}(X_{t, i}, A_{t, i}) - \bar{p}_t(X_{t, i}, A_{t, i}, S_t)}{1}\indicator{\bar{E}_t}} \\
    &\,\leq 
    |\Xset| \E{t}{\max_{x}\abs{T_{M_t}(X_{t, i}, A_{t, i}, x) - \bar{p}_t (X_{t, i}, A_{t, i}, x, S_t)}\indicator{\bar{E}_t}}\,.
\end{align*}
Now, using \cref{lem:beta_dir_sub-gaussianity} for Dirichlet distributions, we have
\begin{align*}
    &\E{t}{\max_{x}\abs{T_{M_t}(X_{t, i}, A_{t, i}, x) - \bar{p}_t (X_{t, i}, A_{t, i}, x, S_t)}\indicator{\bar{E}_t}} \\
    &\,\leq \sum_{(x, a, x')} \int_{p = \phi_t(x, a, S_t) / \sqrt{|\Xset|}}^{\infty} p \prob{t}{\abs{T_{M_t}(x, a, x') - \bar{p}_t(x, a, x', S_t)} = p} dp\\
    &\,\leq \sum_{(x, a, x')} 2\prob{t}{\abs{T_{M_t}(x, a, x') - \bar{p}_t(x, a, x', S_t)} \geq \phi_t(x, a, S_t) / \sqrt{|\Xset|}}\\
    &\,\leq \sum_{(x, a, x')}2\exp\left[ - \frac{\phi_t(x, a, S_t)^2}{2|\Xset|/\left(4\left(\norm{\alpha^T_{k, S_t}(x, a)}{1} + 1\right)\right)}\right] \\
    &\,\leq 1/(2n)\,,
\end{align*}
So, we can bound the first term of \eqref{eq:bayes_regret_decomposition_mdp} by
\begin{align*}
    \E{}{\sum_{t = 1}^n \E{t}{V_{M_t}(\pi_t) - \overline{V}_t(S_t, \pi_t)}}
    &\leq 
    |\Xset| h^2 + h \sum_{t = 1}^n\sum_{i = 1}^h \E{t}{\beta_t(X_{t, i}, A_{t, i}, S_t)} \,.
\end{align*}

\paragraph{Step 2.}
For each episode $t$, we define $C_t$ as follows:
\begin{align*}
    C_t = \left\{s \in \Sset: G_t(s) \leq \sqrt{hN_t(s) \log{n}}\right\}\,,
\end{align*}
where $N_t(s) = \sum_{\ell = 1}^{t-1} \indicator{S_\ell = s}$ is the number of times $s$ was sampled from the posterior and $G_t(s)$ is defined as 
\begin{align*}
    G_t(s) = \sum_{\ell=1}^{k - 1} \indicator{S_\ell = s} \left(\overline{V}_t(\pi_t, s) - h\sqrt{2}\sum_{t = 1}^{h}\beta_t(X_{\ell, t}, A_{\ell, t}, s) - \sum_{t = 0}^{h-1} R_{\ell, t} \right)\,.
\end{align*}

We show that $S_* \in C_t$ holds with high probability for any episode via the following lemma.
\begin{lemma}
\label{lem:gap_concentration_mdp}
For any episode $t$, $\prob{}{S_* \not\in C_t} \leq 2Lhn^{-1}$.
\vspace{-0.05in}
\end{lemma}
\begin{proof}
Fix $S_* = s$. We know that $s \in C_t$ occurs as long as $G_t(s)$ is not too large. Let us define $\mathcal{T}_{t, s} = \{\ell  < t: S_\ell = s\}$ as the episodes where $s$ is sampled until episode $t$. We want to upper-bound $G_t(s)$ by a martingale with respect to history, then bound the probability that $G_t(s)$ is too large using Azuma's inequality for concentration of martingales.

Let us define 
\begin{align*}
\mathcal{E}_{t, i} 
= \big\{&\abs{\bar{r}_t(X_{t, i}, A_{t, i}, S_t) - R_{M_t}(X_{t, i}, A_{t, i})} \leq \sqrt{2}c_t(X_{t, i}, A_{t, i}, S_t) \,,\, \\
&\quad \norm{\bar{p}_t(X_{t, i}, A_{t, i}, S_t) - T_{M_t}(X_{t, i}, A_{t, i})}{1} \leq \sqrt{2}\phi_t(X_{t, i}, A_{t, i}, S_t) \big\}
\end{align*}
as the event that the mean reward and transition probabilities for episode $t$ of episode $k$ are not far from their posterior means. Let $\mathcal{E}= \cap_{t = 1}^n \cap_{i=1}^h \mathcal{E}_{t, i}$ be the event that this holds for all episodes and steps and $\bar{\mathcal{E}}$ be the complement. We know that
\begin{align*}
\prob{}{\bar{\mathcal{E}}} 
\leq 
\sum_{t = 1}^n \sum_{i = 1}^h\sum_{s \in \Sset}\sum_{x, a}\E{}{\prob{t}{\mathcal{E}_{t, i}}}
\leq 
\sum_{t = 1}^n \sum_{i = 1}^h\sum_{s \in \Sset}\sum_{x, a} \left(|\Xset| |\Aset| n\right)^{-2}
\leq Lhn^{-1}\,,
\end{align*} 
where we use that we have $R_{M_t}(x, a)$ and $T_{M_t}(x, a)$ follow a Beta and Dirichlet distribution, respectively, which are sub-Gaussian from \cref{lem:beta_dir_sub-gaussianity}.

For episode $\ell \in \mathcal{T}_{t, s}$, let $Z_\ell = V_*(\pi_\ell) - \sum_{i = 1}^h R_{\ell, t}$. Observe that $\E{\ell}{Z_{\ell}} = 0$, so that $(Z_{\ell})_{\ell \in \mathcal{T}_{t, s}}$ is a martingale difference sequence with respect to histories $(H_\ell)_{\ell \in \mathcal{T}_{t, s}}$. Also note that since $Z_\ell$ the sum of $h$ Bernoulli random variables and is therefore $\sigma^2$-sub-Gaussian with $\sigma^2 = h/4$.
We know that conditioned on $H_\ell$,
\begin{align*}
\overline{V}_\ell(\pi_\ell, s) - h\sqrt{2}\sum_{t = 1}^{h}\beta_\ell(X_{\ell, t}, A_{\ell, t}, s)\indicator{\mathcal{E}_{\ell, t}} - \sum_{t = 0}^{h-1} R_{\ell, t}
\leq 
V_*(\pi_\ell) - \sum_{t = 0}^{h-1} R_{\ell, t}
= Z_\ell
\end{align*}
where we use \cref{lem:mdp_value_difference} to bound $V_*(\pi_\ell) - \overline{V}_\ell(\pi_\ell, s)$. This implies that conditioned on $(H_\ell)_{\ell \in \mathcal{T}_{t, s}}$, we have
\begin{align*}
    G_t(s)\indicator{\mathcal{E}} 
    =
    \sum_{\ell \in \mathcal{T}_{t, s}} \left(\overline{V}_\ell(s, \pi_\ell) - h\sqrt{2}\sum_{t = 1}^{h}\beta_\ell(s, X_{\ell, t}, A_{\ell, t})\indicator{\mathcal{E}_{\ell, t}} - \sum_{t = 0}^{h-1} R_{\ell, t}\right)
    \leq 
    \sum_{\ell \in \mathcal{T}_{t, s}} Z_\ell\,.
\end{align*}

For any episode $t$, we have that $\mathcal{T}_{t, s}$ is a random quantity. First, we fix $|\mathcal{T}_{t, s}| = N_t(s) = u$ where $u < t$ and yield the following due to Azuma's inequality, 
\begin{align*}
  \prob{t}{G_t(s)\indicator{\mathcal{E}} \geq \sqrt{4 (h/4)u \log{n}}} 
  &\leq 
  \prob{}{\sum_{\ell \in \mathcal{T}_{t, s}} Z_\ell \geq  \sqrt{4 (h/4) u \log{n}}}
  \leq \exp\left[-2\log n\right]
  = n^{-2}\,.
\end{align*}
Finally, by the union bound, we have
\begin{align*}
  \prob{}{S_* \not\in C_t}
  &\leq 
  \sum_{s \in \Sset} \sum_{u = 1}^{t - 1}
  \prob{}{G_t(s) \geq \sqrt{h u \log{n}}} \\
  &\leq 
  \prob{}{\bar{\mathcal{E}}} + \sum_{s \in \Sset} \sum_{u = 1}^{t - 1}
  \prob{}{G_t(s)\indicator{\mathcal{E}} \geq  \sqrt{h u \log{n}}}
  \leq 2 L h n^{-1}\,.
\end{align*}
This completes the proof.
\end{proof}

\paragraph{Step 4.}
Following the sketch of \cref{sec:proof_sketch_mdp}, we can rewrite the second term of \eqref{eq:bayes_regret_decomposition_mdp} as
\begin{align*}
\E{}{\sum_{t = 1}^n \E{t}{\overline{V}_t(S_t, \pi_t) -  V_{M_*}(\pi_t)}}
&\leq
\E{}{\sum_{t = 1}^n (\overline{V}_t(S_t, \pi_t) -  V_{M_*}(\pi_t))\indicator{S_t \in C_t}} + h \sum_{t = 1}^n \prob{}{S_* \not\in C_t}\,.
\end{align*}
From Step 1 and 2, the second term is bounded as $2Lh$. Finally, the remaining term can be bounded as
\begin{align*}
    &\E{}{\sum_{t = 1}^n (\overline{V}_t(S_t, \pi_t) -  V_{M_*}(\pi_t)\indicator{S_t \in C_t}} \\
    &\,= 
    h\sqrt{2}\E{}{\sum_{t = 1}^n \sum_{t = 1}^{h} \beta_t(X_{t, i}, A_{t, i}, S_t)} + 
    \E{}{\sum_{t = 1}^n \left(\overline{V}_t(S_t, \pi_t) - h\sqrt{2} \sum_{i = 1}^h \beta_t(X_{t, i}, A_{t, i}, S_t) - \sum_{t=1}^h R_{t, i}\right)\indicator{S_t \in C_t}} \\
    &\,\leq
    h\sqrt{2}\E{}{\sum_{t = 1}^n \sum_{t = 1}^{h} \beta_t(X_{t, i}, A_{t, i}, S_t)} + 
    \E{}{\sum_{s \in \Sset}G_{n + 1}(s) + Lh} \\
    &\, \leq 
    h\sqrt{2}\E{}{\sum_{t = 1}^n \sum_{t = 1}^{h} \beta_t(X_{t, i}, A_{t, i}, S_t)} + 
    \sqrt{Lnh\log{n}} + Lh\,,
\end{align*}
where we use that up until the last episode $t' = \max_{t \in [m]}\{S_t = s\}$ a latent state $s$ is sampled from the posterior, there is an upper-bound on its overestimation $G_{t'}(s)$. 

Let $\Lambda_{0, s} = \min\{\min_{x, a}\norm{\alpha_{0, s}^R(x, a)}{1}, \min_{x, a}\norm{\alpha_{0, s}^T(x, a)}{1}\}$ represent at least how concentrated the reward and transition priors are for latent state $s$. Let $\Lambda_{0, \min} = \min_{s \in \Sset} \Lambda_{0, s}$ be the minimum over latent states. 
What remains the bounding the sum of confidence widths, which is done in \cref{lem:tabular_mdp_sum}. 
Combining the regret due to both terms gives,
\begin{align*}
    \Bregret(m) &\leq 4|\Xset|h\sqrt{2|\Aset|nh\log(4|\Xset||\Aset|n) \log\left(1 + \frac{nh}{2 |\Xset||\Aset|\Lambda_{0, \min}}\right)}
    + 2|\Xset||\Aset|h^2 \\
    &\qquad + 
    \sqrt{Lnh\log{n}} + 3Lh\,.
\end{align*}
\qed

\end{document}